\documentclass[3p,review]{elsarticle}

\pdfoutput=1

\usepackage{lineno,hyperref}
\modulolinenumbers[5]

\usepackage{booktabs}
\usepackage{csquotes}
\usepackage{multirow}
\usepackage{amssymb,amsmath,amsthm,amssymb}
\usepackage{subcaption}
\usepackage[linesnumbered,ruled]{algorithm2e}

\usepackage{tikz}
\usetikzlibrary{positioning}
\usetikzlibrary{arrows}
\usetikzlibrary{calc}
\usetikzlibrary{patterns}
\usetikzlibrary{shapes}
\usetikzlibrary{decorations.text}
\usetikzlibrary{fit}

\newtheorem{plainAlgorithm}{Algorithm}
\newtheorem{theorem}{Theorem}
\newtheorem{corollary}{Corollary}
\newtheoremstyle{case}{}{}{}{}{}{:}{ }{}
\theoremstyle{case}
\newtheorem{case}{Case}

\usepackage{color}

\journal{EURO Journal on Computational Optimization}

\newcommand{\abs}[1]{\left|#1\right|}
\newcommand{\set}[1]{\left\{#1\right\}}

\newcommand{\red}[1]{#1}

\newcommand{\redtwo}[1]{#1}

\newcommand{\resultpath}[1]{#1}

\makeatletter
\def\ps@pprintTitle{%
 \let\@oddhead\@empty
 \let\@evenhead\@empty
 \def\@oddfoot{}%
 \let\@evenfoot\@oddfoot}
\makeatother

\begin{document}

\begin{frontmatter}

\title{\redtwo{Mobility Offer Allocations in Corporate Settings}\tnoteref{ack}}
\tnotetext[ack]{This work has been partially funded by the Climate and Energy Funds (KliEn) within the strategic research program ``Leuchtt{\"u}rme  der Elektromobilit{\"a}t'' under grant number 853767 (SEAMLESS).}

\author{Sebastian Knopp\textsuperscript{*}}
\author{Benjamin Biesinger\textsuperscript{**}}
\author{Matthias Prandtstetter\textsuperscript{***}}
\address{AIT Austrian Institute of Technology \\
	Center for Energy -- Integrated Energy Systems \\
	Giefinggasse 4, 1210 Vienna, Austria. \\
	\vspace{2ex}
	\textsuperscript{*} $\langle$firstName$\rangle$.$\langle$lastName$\rangle$@untis.at \\
	\textsuperscript{**} $\langle$first letter of first name$\rangle$$\langle$lastName$\rangle$@gmail.com \\
	\textsuperscript{***} $\langle$firstName$\rangle$.$\langle$lastName$\rangle$@ait.ac.at
}

\begin{abstract}
Corporate mobility is often based on a fixed assignment of vehicles to employees.
Relaxing this fixation \redtwo{and} including
alternatives such as public transportation or taxis for business and private trips
could increase fleet utilization and foster the use of battery electric vehicles.
\redtwo{We introduce the \emph{mobility offer allocation problem}
as the core concept of a flexible booking system for corporate mobility. The problem}
is \red{equivalent to interval scheduling on dedicated unrelated parallel machines.}
We show that the problem is NP-hard to approximate within any factor.
We describe problem specific conflict graphs
for representing and exploring the structure of feasible solutions.
A characterization of all maximum cliques in these conflict graphs reveals symmetries which
allow to formulate stronger integer linear programming models.
We also present an adaptive large neighborhood search based approach which makes use of conflict graphs as well.
\red{In a computational study,
the approaches are evaluated.
It was found that
greedy heuristics perform best if very tight run-time requirements are given,
a solver for the integer linear programming model performs best on small and medium instances,
and the adaptive large neighborhood search performs best on large instances.}

\end{abstract}


\end{frontmatter}

\section{Introduction}

The transportation sector is facing major changes due to digitalization \redtwo{and} urbanization.
\redtwo{Novel} sharing concepts that offer mobility as a service arise.
While such changes can be observed in many areas, 
corporate mobility has not changed significantly for decades.
\redtwo{Many} companies assign cars to \redtwo{certain employees based on their} hierarchy level.
\redtwo{Usually, such cars} can be used for business and private trips.
This fixed assignment of \redtwo{one car to one} employee
can result in rather large but inefficient fleets
\redtwo{with} cars \redtwo{being} used only less than one hour per day on average~\cite{bates2012spaced,shoup2017high}.
\redtwo{Furthermore}, such cars are often larger than necessary since 
all mobility needs of the employee have to be covered with just one car.
At the same time, other employees are neglected in such concepts.
Often, an additional fleet of vehicles is available for business trips.
\redtwo{These pool cars are booked using the first-come, first-served principle.}

\redtwo{
This paper proposes a novel concept for a corporate mobility service
to improve this situation.
The main idea is to focus on employees' mobility demands
rather than on fixed assignments of cars to employees.
Similar to car rental systems~\cite{oliveira2017fleet},
travelers book mobility rather than a specific car.

A \emph{mobility demand} can represent business travels 
(e.g. travel to a customer location for a meeting from 2:00pm to 4:00pm)
as well as private travels (e.g., weekend trip).
A \emph{mobility offer} is one possible way to satisfy a mobility demand.
For the meeting example from above, one mobility offer is to reserve a pool car from 1:30pm to 4:30pm.
Another mobility offer is to take the train leaving at 1:15pm and returning at 4:45pm.
The final assignment of cars or other transportation modes to travelers
is done via an automatic mobility offer matching system.

This is modeled by the Mobility Offer Allocation Problem (MOAP).
The corresponding optimization model and solution techniques are presented in this paper.
The objective is to determine an integrated allocation that fulfills all mobility demands
and respects the vehicle fleet size
while minimizing expenses and emissions.
}

\subsection{A Corporate Mobility Offer Concept with Flexible Bookings }

\redtwo{The Mobility Offer Allocation Problem} was \redtwo{put into practice as part of an} applied research project
in Austria\footnote{Project SEAMLESS (Sustainable Efficient Austrian Mobility with
Low-Emission Shared Systems), 2016--2019,
\mbox{\url{http://www.seamless-project.at/}},
strategic research program ``Leuchtt{\"u}rme  der Elektromobilit{\"a}t''}.
\redtwo{This paper focuses on modeling and solving the Mobility Offer Allocation Problem.
The research project adressed the corporate mobility concept in a broader way.
The concept aims at improving costs and environmental impacts
by shrinking the fleet \redtwo{size} due to improved utilization,
and by electrifying the fleet due to increased flexibility
(e.g., small electric cars for city trips, large combustion engine cars for holiday trips).
To fully exploit those benefits,
corporate mobility should be available for all employees
and for business as well as for private trips.
The following five main components were addressed
in the implementation of this corporate mobility concept:}

\begin{description}
\item[\bf Fleet of pool cars:] The appropriate mix and size of the fleet has to be determined
based on the company's specific needs.
This fleet might consist of \redtwo{various car types}: from small one- or 
two-seated city cars over vans to small- or medium-sized (light) trucks.
Considering private trips in addition to business trips
could improve company fleet utilization while reducing the need for personally owned cars.
In the literature, a lot of work is dedicated to finding optimal fleet compositions
\cite{hoff2010industrial}.
Battery electric cars should be incorporated \redtwo{to reduce emissions}~\cite{ma2012new}.
Additionally, fleets \redtwo{could be made available} not only to 
employees but also to the public since this demand is often complementary.

\item[\bf Integration of alternative mobility offers:]
\redtwo{The} concept of mobility as a service
\redtwo{should} not be limited to fleet vehicles~\cite{goodall2017rise}.
Alternative modes of transportation
\redtwo{such as} public transportation, bicycle and car sharing systems, or taxis
should be considered as well. 
\redtwo{The system should automatically determine possible alternative mobility offers.}
A seamless integration of such offers could advance the use of low-emission \redtwo{transport} modes.

\item[\bf Flexible booking and accounting system:]
\redtwo{Users} should be able to state preferences in an easy-to-use booking system.
\redtwo{Reserving specific cars should be allowed} only if necessary.
Suitable mobility offers are then assigned automatically.
Additionally, the booking system needs to incorporate an accounting system 
where individual trips can be billed according to their context (e.g., business 
trips are accounted to the company while private trips are billed to the 
traveler).

\item[\bf Key-less access:]
\redtwo{A key-less solution eases access to cars.}
This means each user of the system should \redtwo{have a digital device to
access and drive} company cars.
\redtwo{This could be} an RFID chip card or an application on a mobile phone.

\item[\bf Motivational strategies:]
Appropriate motivational strategies
should be applied in order to achieve the needed mind shift.
Studies have shown that users are hardly willing to waive amenities they are already used to~\cite{gotz2011attraktivitat}.
Economic incentives and
motivational strategies \redtwo{as proposed in}
behavioral economics could be applied~\cite{metcalfe2012behavioural}.
Not only future travelers \redtwo{need to be convinced}, but also upper management, the 
accounting department, or the fleet management division.

\end{description}

\subsection{Related Work}

From an application point of view, only few similar approaches have been found in the literature.
A~related system for sharing electrical vehicles in corporate contexts developed in an applied project
is presented in \citet{Ostermann2014}
which underlines the relevance of the subject at hand.
The method for scheduling vehicles within that system is described in \citet{Koetter2015}.
The proposed approach aims at minimizing fragmentation within the usage of the vehicle fleet
in order to leave room for assigning novel requests.
\redtwo{Their} solution approaches stem from control theory and operating systems,
\redtwo{whereas our paper considers} the problem from an operations research point of view.
\citet{Betz2016} propose an approach which integrates the scheduling of charging battery electric vehicles.
They propose a time indexed mixed integer linear programming (ILP) formulation
\redtwo{with} discrete time periods of 15 minutes.
In this approach, each vehicle class is solved independently.
The authors report computational times of 2~hours for instances \redtwo{with} 30~demands
(called trips in \redtwo{their} paper).
A similar problem, also including charging scheduling, is tackled in \citet{Sassi2014}.
They propose a mixed integer linear programming model and a heuristic approach.
For the exact approach, the authors report computational times of one~hour for instances with about 120~demands
(called tours in \redtwo{their} paper).
Their ILP model avoids overlapping vehicle assignments by including one constraint for each possible vehicle assignment.
\redtwo{We use} a stronger formulation based on cliques in a conflict graph.
Though not explicitly considering the charging \redtwo{process} of vehicles (and thus being less general regarding this aspect),
the problem considered \redtwo{in our paper} is more general \redtwo{with regard to} mobility options.
In particular, vehicle dependent journey intervals (possibly mode of transport dependent)
and alternative journey intervals (e.g., modeling appointment alternatives) are considered,
whereas the approaches from the literature mentioned before assume identical journey intervals for each vehicle.
Also, our approach can include different vehicle admissibility for each demand.

In this paper, the problem is modeled as a generalized operational fixed interval scheduling problem.
\redtwo{Consider~\citet{Kolen2007}} for a survey on interval scheduling and
\redtwo{\citet{kovalyov2007fixed} for} fixed interval scheduling.
\citet{kovalyov2007fixed} define the fixed interval scheduling problem as follows.
We are given $n$ independent non-preemptive jobs to be processed on $m$ independent parallel machines, where each machine can process at most one job at a time.
Each job has a machine-dependent weight and a set of fixed intervals in which it can be processed at a specific machine.
Now, the aim of the \emph{tactical} fixed interval scheduling problem is to minimize the number of needed machines given that all jobs have to be scheduled~\cite{kroon1997exact}.
The goal of the \emph{operational} fixed interval scheduling problem is to maximize the weight of the jobs that can be scheduled on a given set of machines~\cite{ng2014graph}.
In our application, machines correspond to vehicles, jobs to mobility demands, and the possible assignments of a job to a machine at a specific interval to mobility offers.
\citet{ng2014graph} present heuristics for a special case of the operational fixed interval scheduling problem, to which we compare our developed algorithms in Section~\ref{sec:results}.

Similarly, \emph{vehicle scheduling} problems (see~\citet{bunte2009overview} for an overview)
are often found in the context of public transport planning.
\redtwo{They} deal with the task of assigning and sequencing vehicles to trips with fixed travel times.
Some variants of vehicle scheduling account for multiple vehicle types~\cite{hassold2014public} or allow slightly changing the timetable of the trips~\cite{desfontaines2018multiple}.
\redtwo{Classical} vehicle scheduling problems deal with public transport, e.g., bus service planning.
\redtwo{Another} variant is the \emph{rental vehicle scheduling} or \emph{vehicle-reservation assignment} problem.
In these problems, there are usually multiple depots where the fleet is located, there are dynamic booking requests, substitutions (e.g., to a better vehicle class) are possible, and there is the need of vehicle relocations between the depots.
\citet{oliveira2017fleet} give an overview of problems arising in the context of fleet and revenue management of car rental companies.
In~\citet{Ernst2011}, a car rental problem is modeled using a set of assignment problems with linking constraints and tackled using a Lagrangean heuristic.
A real-world use case of such problems is shown in~\citet{Ernst2007} in which a fleet of around 4000 vehicles of a company located in Australia and New Zealand is scheduled. 
Another practical application was tackled in~\citet{oliveira2014relax}
which also presents an integer linear programming model and a matheuristic \redtwo{to solve} the problem.

Interval scheduling problems are based on the concept of interval graphs, in which conflicts between intervals are represented as undirected edges.
Interval graphs are a widely studied graph class in algorithmic graph theory, e.g.,
as a subclass of chordal graphs~\cite{Farber1984}.
Since in our case each interval has a transport mode dependent cost,
the problem is also close to the weighted interval graph coloring problem
which is proven to be NP-hard in \mbox{\citet{Escoffier2005}}.
Another related problem is the maximum weighted independent set problem for interval graphs which is discussed, e.g., in \citet{Bar2001}.
In contrast to the problem considered in our work, the maximum weighted independent set problem is, however, solvable in polynomial-time because of its restriction to interval graphs.
Another related problem is the interval scheduling problem with a resource constraint~\cite{Angelelli2014} which is, in contrast to the MOAP, not a variant of a multiple-interval scheduling problem.
Many problems, like independent set, dominating set, and clique, are shown to be NP-hard
for multiple-interval graphs whereas they are not for 1-interval graphs~\cite{Butman2010}.
Similar models are also used in course timetabling,
consider, e.g., the overview provided in \citet{Burke2010}.
Related are also variants of graph coloring problems~\cite{Marx2004}.

\subsection{Contributions and Structure of this Paper}

We propose a flexible booking system which
leads to the introduction of \red{an} NP-hard \emph{mobility offer allocation problem}.
Solution methods with varying trade-offs between
run-time and solution quality are described and evaluated.
This work complements the existing literature
as it focuses on a real-world application
of multi-interval scheduling which has found only
little attention in the literature.
\red{The mobility offer allocation problem is equivalent to
interval scheduling on dedicated unrelated parallel machines~\cite{ng2014graph}.
To emphasize the application specific focus of this paper,
we use the name \emph{mobility offer allocation problem}.
We will demonstrate its relatation to fixed interval scheduling problems
and show that MOAP is NP-hard to approximate.}
Known heuristics are outperformed by the proposed methods.
\redtwo{We define} problem specific conflict graphs
for representing and exploring the structure of feasible solutions.
We develop a characterization of all maximum cliques
in these conflict graphs,
revealing symmetries which
allow to formulate stronger integer linear programming models.
We also present an adaptive large neighborhood search based approach which makes use of conflict graphs as well.
A computational study using two sets of benchmark instances confirms that,
as one would expect,
\red{the greedy heuristics perform best if very tight run-time requirements are given,
a solver for the integer linear programming model performs best on small and medium instances,
and the adaptive large neighborhood search performs best on large instances.}
The integer linear programming approach of this paper solves instances with up to 200~demands in less than one second.
Instances with up to 2000~demands are solved to optimality within one~hour of computational time.

The outline of this paper is as follows.
Section~\ref{sec:modeling} formally defines the problem and discusses the modeling.
Section~\ref{sec:conflictGraph} defines conflict graphs
as a foundation for the solution approaches proposed in Section~\ref{sec:solutionAlgorithms}.
Then, Section~\ref{sec:vehicleclasses} discusses how symmetries can be exploited
by grouping vehicles into disjoint classes.
A~description of benchmark instances along with a computational study is presented in Section~\ref{sec:results},
followed by conclusions and an outlook in Section~\ref{sec:conclusions}.

\section{Problem Description and Complexity}
\label{sec:modeling}

First, this section formally specifies the problem that is investigated.
Then, it \redtwo{discusses how that model can capture} various practical requirements.
Finally, the section shows that the problem is NP-hard to approximate within any factor.

\subsection{\redtwo{Formal} Problem Description}

In the \emph{Mobility Offer Allocation Problem},
we are given a set of \emph{mobility demands}~$D$ and a fleet of \emph{vehicles}~$V$ representing resources with limited availability.
For each mobility demand~\mbox{$d \in D$},
we are given a set of \emph{mobility offers}~$O_d$
forming the overall set of offers $O=\mathop{\dot{\bigcup}}_{d\in D}O_d$.
Note that for different mobility demands $d_1 \neq d_2 \in D$ we assume the sets $O_{d_1}, O_{d_2}$ to be disjoint.
Each mobility offer $o \in O$ is associated with a \emph{cost}~\mbox{$c_o \in \mathbb{R}$}
and an \emph{journey interval} $T_o = [a_o, b_o)$ with $a_o, b_o\in \mathbb{R}$
defining its start time~$a_o$ and end time~$b_o$ (assuming $a_o < b_o$).
The \emph{duration} of $T_o$ is denoted by $\tau_o = b_o - a_o$.
A mobility offer might require a \emph{vehicle},
so $v_o \in V \cup \{\ast\}$ specifies for each offer $o\in O$ 
either the required vehicle or $\ast$ if no vehicle is needed  (e.g., if the mobility offer corresponds to using public transport).

The problem is to \emph{select} exactly one mobility offer for each demand
such that the total cost of the selected offers is minimal
while overall feasibility is ensured.
Feasibility is achieved if for each pair of selected offers $o_1, o_2 \in O$
with $v_{o_1} = v_{o_2} \not= \ast$ it holds that $T_{o_1}\cap T_{o_2} =\emptyset$,
i.e., the journey intervals of all selected offers
that use the same vehicle do not overlap.

\subsection{Discussion}

Although the problem description given above might appear simplistic,
many features and aspects of practical concerns can be included
\redtwo{due} to the way mobility offers are generated.
We consider this simplicity to be an important advantage of the chosen modeling.
A mobility offer can comprise a complex itinerary consisting of many different locations.
In a real world setting, the offers are not given beforehand, instead they are computed on demand.
Known approaches for route planning (see \cite{Bast2016} for an overview)
can be employed for computing routes for the given vehicles and user needs
to accurately compute journey intervals.
In particular, alternative modes of transports
\redtwo{in addition to} the given fleet of cars \redtwo{can} be included.
This can, e.g., comprise public transport operators,
taxi cooperations, and bike sharing providers.
\red{We assume an infinite capacity for these modes.
Related mobility offers do not require a vehicle from the given fleet (i.e., $v_o = \ast$)
and thus are never in conflict to each other.}
Note that the journey intervals of offers not only include travel times,
but also service times, waiting times, or even visits of multiple customers
with additional travel times in between.
Different offers belonging to the same demand may feature differing journey intervals.
In particular, it is possible to include multiple offers per vehicle
with different journey intervals for modeling alternative dates for the same appointment.
Demands can also be used to model maintenance tasks for the vehicles of the fleet,
with corresponding offers defining possible maintenance dates.

\redtwo{
Multiple offers can require the same vehicle at the same time.
This does not violate the requirement that the offer sets
of different mobility demands must be disjoint.
Consider for example two demands $A$ and $B$ which
both could be satisfied by using a vehicle~$v$.
In this example, we have an offer $o_A$ whose selection would mean
the car will be assigned to satisfy demand $A$,
and a different offer $o_B$ whose selection would mean
the car will be assigned to satisfy demand $B$.
In this case, at most one of these offers can be selected in a feasible solution.
}

A core assumption is that each vehicle is assigned to a fixed location
where all trips start and end.
Yet, different vehicles can be located at different places.
This assumption holds in many corporate contexts and
also in related use-cases such as, e.g., car sharing systems
\red{where each vehicle is bound to one fixed location}.
Beside the administrative reasons observable in practice (e.g., maintenance responsibilities),
station based systems have some inherent advantages:
First, there is no need to consider vehicle relocations due to the imbalance of travel demand.
Second, station based systems are robust against canceled trips.
Canceling a trip in advance cannot cause a problem
if vehicle locations are fixed,
but canceling a trip from $A$ to $B$ in a more flexible system
could render other trips starting from $B$ impossible.
Certainly, there are scenarios where allowing relocations is reasonable.
However, such scenarios are not considered in this work.

Especially in practical applications,
feasibility of the problem instances cannot always be ensured.
A reasonable assumption, however, is to add an offer to each demand denoting a taxi trip which has high cost but is always feasible.
When taxis are not available, an artificial offer~\mbox{$o \in O$}
representing a regret cost can be added to each demand denoting that this demand cannot be fulfilled.
Therefore, infeasibilities are not explicitly considered in the proposed solution algorithms.

\subsection{Complexity}

The mobility offer allocation problem is equivalent to the
interval scheduling on dedicated unrelated parallel machines~(ISDU)
introduced in~\cite{ng2014graph}.
A mapping between instances of ISDU and instances of MOAP
is given by identifying machines with vehicles,
jobs with mobility demands,
and intervals with mobility offers.
Unavailability intervals can be represented by
artificial mobility demands with a single associated offer.
ISDU is shown to be NP-hard in~\cite{ng2014graph}
by stating that it is a generalization of interval scheduling on dedicated identical parallel machines (ISDI),
which is shown to be NP-hard by~\cite{Arkin1987}.
Throughout this paper we use the notation
introduced for the mobility offer allocation problem
as this allows us to discuss in terms of the problem domain.
In addition,
Theorem~\ref{th:np} relates MOAP to another problem
from the literature and shows the stronger result that
MOAP is NP-hard to approximate within any factor by reducing
the interval scheduling problem with machine availabilities (ISMA)
to the mobility offer allocation problem (MOAP).
These results suggests that heuristics might be necessary for solving large and difficult instances.

\begin{theorem}
\label{th:np}
If $P \neq NP$, then for any factor there is no polynomial-time approximation algorithm for the mobility offer allocation problem.
\end{theorem}

\begin{proof}
We show that the MOAP is NP-hard to approximate within any factor by a polynomial-time reduction from the Interval Scheduling with Machine Availabilities (ISMA) problem which has been shown to be NP-complete in~\cite{Kolen2007}.
An instance of the ISMA problem is defined as follows:
There are $m$ machines, continuously available in $[a_i,b_i]$ with $i=1,\dots,m$ and $n$ jobs requiring processing from $s_j$ to $f_j$ with $j=1,\dots,n$.
The question is whether a feasible schedule exists such that each job is processed by a machine within its availability interval such that no two jobs overlap.
From that we construct an instance of MOAP by creating $n$ mobility demands and $m$ vehicles.
For each mobility demand $d_j$ with $j=1,\dots n$ there is a corresponding offer $o$ for each vehicle $m_i$ with $i=1,\dots, n$ if $[s_j,f_j] \subseteq [a_i,b_i]$ with $c_o=0$ and $T_o=[s_j,f_j]$.
Additionally, there is an offer $o'$ for each demand $d_1,\dots,d_n$ with $c_{o'}=1$ and $v_{o'} = \ast$.
Then, the result of the question whether there exists a feasible allocation of exactly one offer to each demand with cost smaller than~1 is also an answer to the ISMA problem.
Now, assume there exists a polynomial-time algorithm for MOAP
which is able to find a solution within a factor of $\alpha$ to the optimal solution.
In case there exists a feasible schedule for the given ISMA problem
the objective function value found by that approximation algorithm must be zero.
Thus, the polynomial-time approximation algorithm for MOAP would solve the ISMA problem.
Unless $P \neq NP$, such an algorithm cannot exist.
\end{proof}

\section{Conflict Graphs}
\label{sec:conflictGraph}

Conflict graphs are a well-known modeling technique, used,
e.g., for solving coloring or scheduling problems.
\redtwo{They} are a fundamental concept for the solution approaches proposed in this paper.
In particular, interval graphs in interval scheduling as discussed in~\citet{Kolen2007} are conflict graphs.
An interval graph is an undirected graph whose nodes correspond to intervals on the real number line
and whose edges identify overlaps between the intervals.
Similarly, we define an offer conflict graph for identifying all possible conflicts between mobility offers.
Nodes in the conflict graph correspond to mobility offers.
Edges identify pairs of offers that may not be selected at the same time.
Subsequently, based on the offer conflict graph,
cliques in that offer conflict graph are identified and a demand conflict graph is introduced.

\subsection{Offer Conflict Graphs}
\label{ssec:offerConflictGraph}

First, for each vehicle~\mbox{$v\in V$}, a conflict graph $G_O^v = (O_v,E_v)$ is defined.
Its nodes $O_v =\set{o \in O \mid o_v = v}$ correspond to all offers requiring vehicle~$v$.
The edges of the graph \mbox{$E^v = \set{\set{o, o'} \in O_v \times O_v \mid T_o \cap T_{o'} \not= \emptyset}$}
identify mobility offers that use the same car and have overlapping journey intervals.
Each conflict graph~$G_O^v$ is an interval graph.
Secondly, we have for each mobility demand~\mbox{$d\in D$} a conflict graph $G_O^d = (O_d, E_d)$
with $E_d = \set{\set{o, o'} \mid o, o' \in O_d}$.
Each conflict graph~$G_O^d$ is a complete graph since exactly one offer must be chosen for each demand.
Then, the \emph{offer conflict graph}~$G_O = (O, E)$ is defined by setting $E = \bigcup_{d\in D}E_d \, \cup \, \bigcup_{v\in V}E^v$.

\subsection{Cliques in Offer Conflict Graphs}

In a general undirected graph $G=(V,E)$, a subset of nodes $C\subseteq V$ is a \emph{clique}
if and only if there exists an edge between \red{each pair of} nodes in $C$, i.e., $\forall \, c_1,c_2\in C : \set{c_1,c_2} \in E$.
A clique $C$ is a \emph{maximum clique} if $G$ does not contain another clique $K$, such that $C\subset K$.
For general graphs, the number of maximum cliques can be exponential in the number of nodes~\cite{Moon1965}.
However, all maximum cliques in the offer conflict graph can be enumerated efficiently
and their number is at most $\abs{D} + \sum_{v \in V} \abs{O_v}$.
This is shown in the following.

\begin{theorem}
\label{th:conflictGraphCliques}
If $K \subset V$ is a maximum clique in a mobility offer conflict graph, then
\begin{itemize}
	\item[-] $\exists \, d \in D$ such that $K = O_d$, or
	\item[-] $\exists \, v \in V$ such that $K$ is a maximum clique in $G_O^v$.
\end{itemize}
\end{theorem}

\begin{proof}
Let $K \subset V$ be a maximum clique.
\begin{case}
If $\exists \, d \in D$ such that $K \subseteq O_d$, then $K = O_d$ since $G_O^d = (O_d, E_d)$ is a complete graph.
\end{case}

\begin{case}
If $\nexists \, d \in D$ such that $K \subseteq O_d$, then
$\exists \; c \neq d \in D$ with $K \cap O_c \neq \emptyset$ and $K \cap O_d \neq \emptyset$.
Further, there must exist an edge $\set{a, b} \in E$ with $a \in O_c$ and  $b \in O_d$ since $K$ is a clique.
By construction of the conflict graph,
this edge can be only induced by a vehicle conflict because $c$ and $d$ belong to different demands.
Thus, we have $v_a = v_b$.
Now, denoting $v = v_a = v_b$, assume there is a node $k \in K$ with $v_k \neq v$.
Since there is no vehicle conflict,
$k$ can be connected to both nodes, $a$ and $b$, only due to a demand conflict.
However, there cannot be a demand conflict to both nodes at the same time,
since they correspond to different demands.
Since $K$ is a maximum clique,
there cannot be a node $k \in K$ with $v_k \neq v$, thus $K \subseteq O_v$.%
\end{case}%
\end{proof}

Theorem~\ref{th:conflictGraphCliques} shows that there are only two types
of maximum cliques in the mobility offer conflict graph.
Next, we describe the construction of the conflict graph by enumerating the cliques identified above.
The first type of cliques (offers belonging \redtwo{to} the same demand) can be derived directly from the problem instance.
The second type of cliques (offers belonging to the same vehicle with overlapping time intervals)
can be computed independently for each vehicle.
As proposed in \cite{Gupta1982}, this can be done by adapting
the algorithm of \cite{Gupta1979} for finding a minimum coloring of an interval graph.
The following algorithm describes that adaption.

\begin{plainAlgorithm} \label{alg:maxClique}
This algorithm successively reports all maximum cliques in an interval graph $G = (V, E)$.
We assume the interval graph to be given by its implicit representation, i.e., a set of intervals in the real line.
The start and end dates of the intervals are denoted as left and right endpoints, respectively.

\begin{enumerate}
	\item Maintain an initially empty set of nodes $C$, representing the current maximum clique candidate.
	\item
		Sort the $2 \cdot \abs{V}$ endpoints of the intervals of $V$ in ascending order.
		In case of ties with left and right endpoints of journey intervals, right endpoints always come first.
	\item Scan the list of sorted endpoints. Let $e$ be the current endpoint.
	\begin{enumerate}
		\item[If]$e$ is a left endpoint: Add the corresponding offer to $C$.
		\item[If]$e$ is a right endpoint:
		\begin{enumerate}
			\item[1.] If the previous endpoint was a left endpoint, report $C$ as a newly found maximum clique.
			\item[2.] Remove $e$ from the current maximum clique candidate $C$.
		\end{enumerate}
	\end{enumerate}
\end{enumerate}
\end{plainAlgorithm}

The runtime complexity of Algorithm~\ref{alg:maxClique} is $\mathcal{O}(\abs{V} \cdot log \, \abs{V})$,
not including the output complexity of reporting newly found maximum cliques.
The number of maximum cliques in an interval graph $G=(V,E)$ is at most $\abs{V}$
since for each node in Algorithm~\ref{alg:maxClique} at most one maximum clique is reported.
Since each clique can contain at most~$\abs{V}$ nodes,
the runtime complexity of the algorithm including the effort for reporting all maximum cliques is $\mathcal{O}(\abs{V}^2)$.

Now, we can make the following observation regarding the number of maximum cliques.

\begin{corollary}
The number of maximum cliques in a mobility offer conflict graph is at most $\abs{D} + \sum_{v \in V} \abs{O_v}$.
\end{corollary}

\begin{proof}
This follows from Theorem~\ref{th:conflictGraphCliques} and
the fact that the number of maximum cliques in an interval graph~\mbox{$G=(V,E)$} is at most $\abs{V}$.
\end{proof}

\subsection{Demand Conflict Graphs}
\label{ssec:demandConflictGraph}

We now introduce demand conflict graphs in order to identify potentially conflicting mobility demands.
These graphs are applied in Section~\ref{ssec:LNS}
for defining an efficient destroy operator
of a large neighborhood search based approach.
The following definition of a demand conflict graph is described in terms of the quotient graph of an offer conflict graph.
Beforehand, we recall the definition of a quotient graph which is a known concept from graph theory (see, e.g., \cite{Gustin1963} or \cite{Schulz2013}).

A \emph{quotient graph} \mbox{$G_q = (V_q, E_q)$} is defined for
a given partitioning $V = \mathop{\dot{\bigcup}}_{i=1}^k \, V_i$ of the nodes of an original graph $G = (V, E)$.
The nodes~$V_q$ of the quotient graph \redtwo{are} then given by
$V_q = \set{V_1, V_2, \dots, V_k}$.
There exists an edge~$\set{V_i, V_j} \in E_q$ between two nodes $V_i, V_j \in V_q$ in the quotient graph~$G_q$
if and only if there exists an edge~$\set{a, b} \in E$ with $a \in V_i$ and $b \in V_j$ in the original graph~$G$.

A \emph{demand conflict graph} $G_D = (D, E_D)$ is defined as the quotient graph of an offer conflict graph $G_O = (O, E)$
with a partitioning of the offers $O = \mathop{\dot{\bigcup}}_{d \in D} \, O_d$ given by the demands $d \in D$.
So, each node in a \emph{demand conflict graph} $G_D = (D, E_D)$ is associated to a mobility demand $d\in D$.
The set of edges $E_D$ contains an edge $\set{d,h}$ with $d, h \in D$ if
there is a potential conflict between the corresponding demands $d \in D$ and $h \in D$.
A potential conflict between demands $d$ and $h$ exists if there is an offer $a \in O_{d}$ conflicting with an offer $b \in O_{h}$,
i.e., both offers have overlapping journey intervals requiring the same vehicle.
Note that demand conflict graphs are related to multiple-interval graphs as described in~\cite{Butman2010}.
The perspective of demand conflict graphs also allows to
see the MOAP as a weighted variant of the known list coloring problem described in \cite{Marx2004},
with colors corresponding to vehicles.

\subsection{Example}

\begin{figure}

\vspace{0.5em}

\tikzset{
	journeyInterval/.style = {
		draw=black,
		fill=white,
		font=\small
	},
	demandEdge/.style = {
		-,
		dashed
	}
}

\newcommand{\drawInterval}[4]{
	\draw[journeyInterval] (#1\linewidth,-#3-0.4) rectangle (#2\linewidth,-#3+0.4);
	\node[] at ({(#1+#2)/2*\linewidth}, -#3) { #4 };
}

\newcommand{\drawSeparatorline}[1]{
	\draw (0, -#1) edge[-, dotted] (\linewidth, -#1) ;
}

\newcommand{\conflictGraphNode}[3]{
	\node (#1) [circle, draw=black, font=\small] at(#2\linewidth, -#3) { #1 };
}

\newcommand{\thickBackgound}[5]{
	\node [circle, fill=#5, minimum size=2.65em] at (#1\linewidth, -#2) {} ;
	\draw[draw=#5, line width=2.65em, rounded corners=0.5em]
		(#1\linewidth, -#2) -- (#3\linewidth, -#4) ;
	\node [circle, fill=#5, minimum size=2.65em] at (#3\linewidth, -#4) {} ;
}

\begin{subfigure}{0.45\linewidth}
    \begin{tikzpicture}[thick,scale=1.0, every node/.style={scale=1.0}]
			\node[] at(0.03\linewidth, -2) {$V_1$};
			\drawInterval{0.125}{0.325}{1}{A1}
			\drawInterval{0.7}{0.9}{1}{A2}
			\drawInterval{0.35}{0.75}{2}{C1}
			\drawSeparatorline{2.75}

			\node[] at(0.03\linewidth, -4) {$V_2$};
			\drawInterval{0.1}{0.35}{3.5}{A3}
			\drawInterval{0.65}{0.95}{3.5}{B1}
			\drawInterval{0.275}{0.825}{4.5}{C2}
			\drawSeparatorline{5.25}

			\node[] at(0.03\linewidth, -6.5) {$\ast$};
			\drawInterval{0.1}{0.35}{6}{A4}
			\drawInterval{0.7}{1.0}{6}{B2}
			\drawInterval{0.25}{0.85}{7}{C3}

			\node (T) at (0.95\linewidth, -8.125) { \small Time } ;
			\draw[thick,-latex] (0.0, -8.125) -- (T) ;

    \end{tikzpicture} \\[-1em]
	\caption{Journey Intervals of an MOAP instance.}
	\label{subfig:journeyIntervals}
\end{subfigure}
\hspace{0.1\linewidth}
\begin{subfigure}{0.45\linewidth}
    \begin{tikzpicture}[thick,scale=1.0, every node/.style={scale=1.0}]

			\draw[draw=black!10, line width=0.75em, rounded corners=0.25em, dotted]
				(0.2\linewidth, -2.5) -- (0.5\linewidth, -3.35) ;
			\draw[draw=black!10, line width=0.75em, rounded corners=0.25em, dotted]
				(0.5\linewidth, -3.15) -- (0.8\linewidth, -4.75) ;

			\thickBackgound{0.2125}{0.975}{0.8}{1}{black!22}
			\thickBackgound{0.2125}{0.975}{0.225}{3.5}{black!22}
			\thickBackgound{0.225}{3.5}{0.225}{6}{black!22}

			\thickBackgound{0.55}{2}{0.55}{4.5}{black!22}
			\thickBackgound{0.55}{4.5}{0.55}{7}{black!22}

			\thickBackgound{0.8}{3.5}{0.85}{6}{black!22}

			\conflictGraphNode{A1}{0.225}{1}
			\conflictGraphNode{A2}{0.8}{1}
			\conflictGraphNode{C1}{0.55}{2}
			\draw (C1) edge[-] (A2);

			\conflictGraphNode{A3}{0.225}{3.5}
			\conflictGraphNode{B1}{0.8}{3.5}
			\conflictGraphNode{C2}{0.55}{4.5}
			\draw (A3) edge[-] (C2);
			\draw (C2) edge[-] (B1);

			\conflictGraphNode{A4}{0.225}{6}
			\conflictGraphNode{B2}{0.85}{6}
			\conflictGraphNode{C3}{0.55}{7}

			\draw (A1) edge[demandEdge] (A2) ;
			\draw (A1) edge[demandEdge] (A3) ;
			\draw (A1) edge[demandEdge, bend right = 30] (A4) ;
			\draw (A2) edge[demandEdge, bend right = 26] (A3) ;
			\draw (A2) edge[demandEdge, bend left = 3] (A4) ;
			\draw (A3) edge[demandEdge] (A4) ;

			\draw (B1) edge[demandEdge] (B2);

			\draw (C1) edge[demandEdge] (C2);
			\draw (C1) edge[demandEdge, bend left = 30] (C3);
			\draw (C2) edge[demandEdge] (C3);

			\node[text opacity=0] (T) at (0.95\linewidth, -8.125) { \small Time } ;
    \end{tikzpicture} \\[-1em]
	\caption{Offer and demand conflict graphs.}
	\label{subfig:conflictGraph}
\end{subfigure}
\vspace{0.5em}
\caption{An example showing vehicles, demands, offers, journey intervals, and corresponding conflict graphs.}
\label{fig:example}
\vspace{2em}
\end{figure}
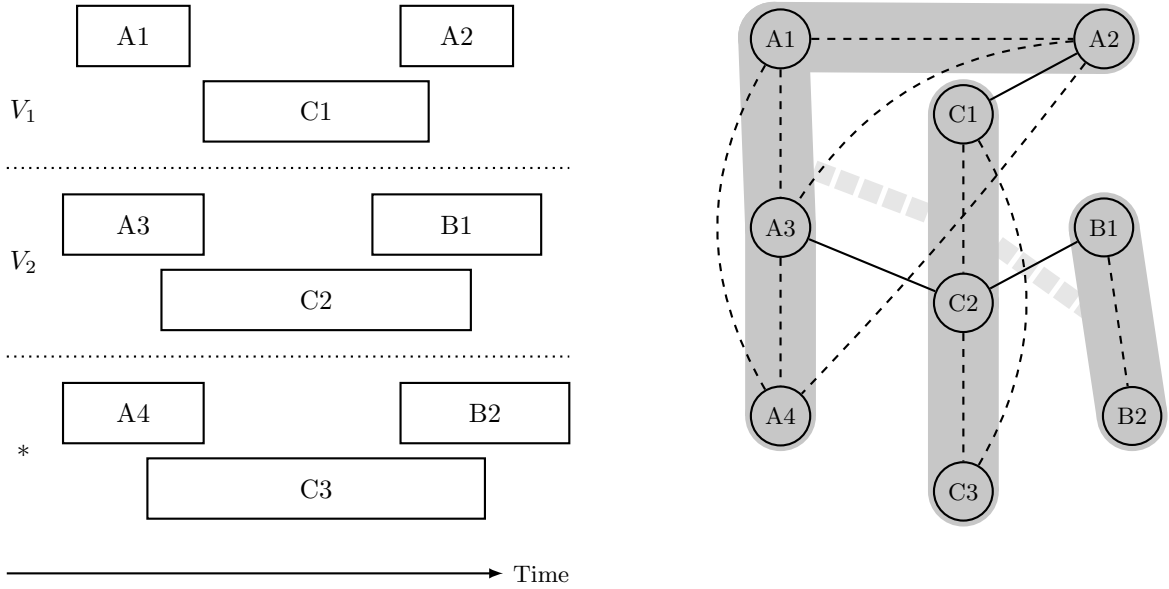

Figure~\ref{fig:example} shows an example of a MOAP instance with two vehicles $V_1$ and $V_2$.
It includes three mobility demands $A$, $B$, and $C$
with mobility offers $\set{A1, A2, A3, A4}$,  $\set{B1, B2}$, and  $\set{C1, C2, C3}$, respectively.
The journey intervals of these offers are shown in Figure~\ref{subfig:journeyIntervals}.
Note that demand $A$ demonstrates alternative appointment dates since
the offers $A1$ and $A2$ require the same vehicle during different journey intervals.
The corresponding conflict graphs are shown in Figure~\ref{subfig:conflictGraph}.
Circles represent nodes in the offer conflict graph.
Conflict edges between two offers of the same demand are drawn using dashed lines.
Conflicts related to vehicles with overlapping journey intervals are drawn using straight lines.
As indicated by the asterisk~$\ast$, the offers $A4$,  $C3$, and $B2$ do not require any vehicles.
\redtwo{Thus}, there are no vehicle induced conflicts in between them.
For the offers which require a vehicle,
we observe overlapping journey intervals between $A2$ and $C1$,  $A3$ and $C2$, and $C2$ and $B1$.
Feasible selections, e.g., are $\set{A3, B1, C1}$ or $\set{A2, B2, C2}$.
The demand conflict graph is indicated by background shapes which group nodes that correspond to the same demand.
These three shapes form the nodes of the demand conflict graph.
\red{For example, the set of nodes $\set{A1, A2, A3, A4}$ in the offer conflict graph
corresponds to the node representing demand $A$ in the demand conflict graph.}
There are two edges in the demand conflict graph
which are indicated by the large dotted lines connecting the shapes.
\red{One edge connects the nodes corrsponding to the demand $A$ and $C$,
another connects the demands $B$ and $C$.}

\section{Solution Approaches}
\label{sec:solutionAlgorithms}

We propose a variety of solution approaches including an ILP model using
a general purpose solver, greedy algorithms, and an adaptive large neighborhood search.
\redtwo{All are} evaluated against two sets of benchmark instances in Section~\ref{sec:results}.
For solving the ILP,
the mixed integer linear programming solver IBM ILOG CPLEX Optimizer, version 12.6.2, is used.
For larger instances,
this exact approach is not successful anymore which is why we also propose heuristic solution algorithms.
\redtwo{All} proposed methods make use of the conflict graphs introduced in Section~\ref{sec:conflictGraph}.

\subsection{Integer Linear Programming Model}
\label{ssec:ilp}

The definition of the offer conflict graph leads to an
integer linear programming model which is presented next.
We define binary decisions variables $x_o\in \{0,1\}$ for each mobility offer $o \in O$
denoting whether or not an offer is selected.
We denote the set of maximum cliques in the conflict graph~$G^v$ of a vehicle $v \in V$ by $C^v$.

\begin{align}
	  min \sum_{o \, \in \, O} x_o \cdot c_o & & \label{ilp:obj}\\
s.t.	\sum_{o \, \in \, O_d} x_o \, = & ~1 & \forall \, d \in D \label{ilp:c1}\\
	\sum_{o \, \in \, K} x_o \, \leq & ~1 & \forall \, v \in V, \; \forall \, K \in C^v \label{ilp:c2}\\
	x_o \, \in & ~\set{0,1} & \forall \, o\in O
\end{align}

\vspace{1em}
The objective function (\ref{ilp:obj}) minimizes the sum of the costs of selected offers.
Constraints (\ref{ilp:c1}) ensure that for each mobility demand exactly one offer is selected.
Constraints (\ref{ilp:c2}) prevent selecting offers which require the same vehicle at the same time.
Constraints~(\ref{ilp:c1}) and (\ref{ilp:c2}) directly correspond to
the two types of maximum cliques identified in Theorem~\ref{th:conflictGraphCliques} of Section~\ref{ssec:offerConflictGraph}.

Note that one could replace \redtwo{Constraints}~(\ref{ilp:c2}) by individual constraints of the form $x_a + x_b \leq 1$
for each edge~\mbox{$\set{a, b}$} of the conflict graphs~$G^v$.
However, the maximum clique based formulation (which uses fewer constraints)
is stronger than the edge based formulation.
\redtwo{It} implies all edge based constraints,
and the solution space of the linear programming relaxation is smaller.
This observation was confirmed in preliminary numerical experiments:
Much larger memory requirements and higher solution times were observed for the edge based formulation
when compared to the maximum clique based formulation.

\subsection{Greedy Heuristic}
\label{ssec:greedy}

\redtwo{We propose} a greedy heuristic
which iterates over all mobility demands and
selects a feasible (i.e., non-conflicting) mobility offer with minimum cost in each iteration.
The set of already selected offers during iteration is denoted by $O' \subset O$.
For a demand $d \in D$ without \redtwo{any} selected offer,
its offers eligible for selection are given by $L_d = \set{o \in O_d \mid \forall v \in O' : (o, v) \notin E)}$.
The offer to select is arbitrarily chosen from $\mathop{\mathrm{argmin}}_{o \in L_d}(c_o)$.
\redtwo{Hence,} the order of the iteration is an important choice which strongly influences the overall objective value.
Therefore, we propose the following sort criteria where
a mobility demand $l\in D$ is chosen before a mobility demand $r\in D$ if

\vspace{0.6em}
\begin{description}
	\item[MinMinCost]: $\displaystyle \min_{o \, \in \, O_l}{c_o} \; < \; \min_{o \, \in \, O_r}{c_o}$,
	\item[MaxMinCost]: $\displaystyle \min_{o \, \in \, O_l}{c_o} \; > \; \min_{o \, \in \, O_r}{c_o}$,
	\item[MinMinCostPerTime]: $\displaystyle \min_{o \, \in \, O_l}{\frac{c_o}{ \tau_o}} \; < \; \min_{o \, \in \, O_r}{\frac{c_o}{\tau_o}}$,
	\item[MaxMinCostPerTime]: $\displaystyle \min_{o \, \in \, O_l}{\frac{c_o}{ \tau_o}} \; > \; \min_{o \, \in \, O_r}{\frac{c_o}{\tau_o}}$,
	\item[MinAveCost]: $\displaystyle \frac{1}{|O_l|} \cdot \sum_{o \, \in \, O_l}{c_o} \; < \; \frac{1}{|O_r|} \cdot \sum_{o \, \in \, O_r}{c_o}$,
	\item[MaxAveCost]: $\displaystyle \frac{1}{|O_l|} \cdot \sum_{o \, \in \, O_l}{c_o} \; > \; \frac{1}{|O_r|} \cdot \sum_{o \, \in \, O_r}{c_o}$.
	\item[Random]: The ordering is determined by random sampling without replacement.
\end{description}
\vspace{0.6em}

This leads to seven different variants of the greedy algorithm.
They are used stand-alone, for initial solution generation, and as a part of the repair methods of the~ALNS.
In detail, the greedy offer selection algorithm works as follows.

\begin{plainAlgorithm} \label{alg:greedy}
This algorithm greedily selects one mobility offer~$o \in O$ for each of the given mobility demands~$D$
using the offer conflict graph $G = (O, E)$.
\begin{enumerate}
	\item
		Sort all mobility offers $o \in O$ lexicographically:
		First, by the position of the corresponding demand according to one of the sorting criteria introduced above;
		Second, ascending by the cost~$c_o$ of the offer.
	\item
		Mark each offer as \emph{selectable}.
	\item Scan the list of sorted offers:
	\begin{itemize}
		\item[If] the current offer $o \in O$ is selectable:
			\begin{itemize}
				\item Report the offer $o$ as selected.
				\item Mark all offers $o'$ with $\set{o, o'} \in E$ as not selectable.
		\end{itemize}
	\end{itemize}
\end{enumerate}
\end{plainAlgorithm}

Algorithm~\ref{alg:greedy} has a runtime complexity of
$\mathcal{O}(\abs{O} \cdot \mathrm{log}\abs{O} + \abs{E})$
since the offers~$O$ are sorted and iterated once,
and each edge~$\set{o, o'} \in E$ is touched at most once.

\subsection{Adaptive Large Neighborhood Search}
\label{ssec:LNS}
We propose an \emph{adaptive large neighborhood search}~(ALNS) metaheuristic,
whose foundation has originally been introduced by~\citet{Shaw1998}.
\redtwo{It} was further developed by \cite{Pisinger2010} who introduced an adaptive choice of its operators in~\cite{ropke2006adaptive}.
Large neighborhood search approaches have already been successfully used for
heuristically solving combinatorial optimization problems
in the domain of vehicle routing~(\cite{Prescott-Gagnon2009, Ribeiro2012}),
pickup-and-delivery problems~(\cite{Ropke2006}), \redtwo{or} scheduling problems~(\cite{Godard2005}).

An outline of the adaptive large neighborhood search is shown in Algorithm~\ref{alg:alns}.
The ALNS utilizes \emph{destroy} and \emph{repair} operators, $\Omega^-$ and $\Omega^+$, respectively.
\redtwo{It} starts by creating an initial solution $s$ followed by a series of destroy and repair moves to improve solutions.
A new solution $s^\mathrm{t}$ is obtained from a previous solution
by a move composed of a pair of destroy and repair operations.
This aims at improving a given solution by unassigning a set of 
decisions variables (destroy) and subsequently reassigning them (repair).
Whenever such a move is \emph{accepted}, e.g., if it improves the current solution, it is the new incumbent solution for the next iteration.
In our ALNS implementation a solution $s=(o_1,\dots,o_{|D|})$ is represented as
\redtwo{a} list of selected offers with $o_i\in O_i$, $\forall i=1,\dots,|D|$.
The objective value $c(s)=\sum_{o\in s}c_o$ of a solution $s$ is the total cost of its offers.
While the standard LNS \redtwo{has} only one destroy and one repair operator, the ALNS extension allows for multiple destroy and repair operators.
In each iteration, one of each operators is chosen based on \emph{weights} $\rho^-$ and $\rho^+$ assigned to the destroy and repair operators, respectively.
The destroy operator  $\omega_i\in\Omega^-$ is chosen with probability $p^-_i=\rho^-_i / (\sum_{i\in \Omega^-} \; \rho^-_i)$.
The repair operator is chosen analogously.
While the LNS \redtwo{accepts} only improving solution candidates, for the acceptance criterion of the ALNS we use a \emph{simulated annealing} based approach as suggested in~\cite{ropke2006adaptive}.
A generated solution $s^\mathrm{t}$ is accepted with probability $exp({-\frac{c(s^\mathrm{t}) - c(s)}{T}})$, where $T>0$ is the temperature.
The temperature starts with an initial value $T_{\mathrm{start}}$ and decreases after each iteration to $T := c \cdot T$ by using a cooling rate $c$ with $0<c<1$. 
As described in~\cite{ropke2006adaptive}, setting the start temperature is crucial for the algorithm's performance, but depends on the problem instance.
Therefore, the same method as in~\cite{ropke2006adaptive} is used.
The start temperature is set such that the first generated solution after the initial solution that is $w$ percentage worse is accepted with probability $p^\mathrm{w}$, where $w$ and $p^\mathrm{w}$ are parameters of the ALNS.
The weights of the selected destroy and repair operators with indices $i$ and $j$ are adjusted after each iteration by setting $\rho^-_i=\lambda\rho^-_i+(1-\lambda)\sigma$ and $\rho^+_j=\lambda\rho^+_j+(1-\lambda)\sigma$.
The parameter $\lambda$ with $0<\lambda<1$ is a \emph{decay} parameter determining the impact of the previous weight value.
The value $\sigma$ modifies the weight depending on the performance of the destroy and repair operation pair.
It is set in the following way ($\sigma_1, \sigma_2$, and $\sigma_3$ are parameters of the ALNS):
\begin{itemize}
	\item $\sigma=\sigma_1$ if $c(s^\mathrm{t}) < c(s^\mathrm{b})$, i.e., the new solution candidate improves the best found solution~$s^\mathrm{b}$.
	\item $\sigma=\sigma_2$ if $c(s^\mathrm{b}) < c(s^\mathrm{t}) < c(s)$, i.e., the new solution candidate improves the current solution.
	\item $\sigma=\sigma_3$ if $s^\mathrm{t}$ is accepted but $c(s) < c(s^\mathrm{t})$, i.e., the new solution does improve neither the best nor the current solution but is still accepted due to the acceptance criterion.
\end{itemize}
Additionally, $\sigma=0$ if the new solution candidate has already been generated, i.e., is a duplicate.
Duplicate checking is implemented by using a hash set storing all \redtwo{hashes of} generated solution candidates.
Furthermore, since the exact repair method (see Section~\ref{sssec:repair}) is expected to consume much more time than the greedy repair method, $\sigma$ is scaled by the time needed for performing the corresponding repair operation.
This reduces the bias towards strong but time-consuming operators.
Corresponding destroy and repair operators are described in the following sections.

\setcounter{algocf}{\value{plainAlgorithm}}

\begin{algorithm}
\caption{Adaptive large neighborhood search~(\cite{Pisinger2010})} \label{alg:alns}
\KwIn{feasible solution $s$}
\KwOut{best found solution $s^\mathrm{b}$}
$s^\mathrm{b}=s$; $\rho^-=(1,\dots,1)$; $\rho^+=(1,\dots,1)$\;
\Repeat{stop criterion is met}{%
select destroy and repair operators $d\in \Omega^-$ and $r\in\Omega^+$ using $p^-$ and $p^+$\;
$s^\mathrm{t}=r(d(s))$\;
\If{accept($s^\mathrm{t},s)$}{
$s=s^\mathrm{t}$\;}
\If{$c(s^\mathrm{t})<c(s^\mathrm{b})$}{
$s^\mathrm{b}=s^\mathrm{t}$\;}
update $\rho^-$ and $\rho^+$\;
}%
\Return $s^\mathrm{b}$\;
\end{algorithm}

\subsubsection{Destroy Operators}
All destroy operators deselect a certain number of selected offers.
The destroy operators are parametrized by a relative size $0<r^{\mathrm{des}}<1$,
which determines the amount of mobility offers to be deselected.
We propose three different approaches:

\begin{description}

	\item[Random Destroy]
		This operator deselects mobility offers which are chosen uniformly at random from the set of selected offers.

	\item[Time Interval Destroy]
	This operator deselects offers within a certain time interval.
	The idea is to remove offers whose journey intervals are close to each other or overlapping.
	The absolute size of the time interval is
	$r^{\mathrm{des}}_{\mathrm{a}} = \lceil (\max_{o\in O}{b_o} - \min_{o\in O}{a_o}) \cdot r^{\mathrm{des}} \rceil$ and the start of the interval $t_a$ is chosen uniformly at random from $[\min_{o\in O}{a_o}, \max_{o\in O}{b_o}]$.
	In case $t_a+r^{\mathrm{des}}_{\mathrm{a}} > \max_{o\in O}{b_o}$, the time interval goes beyond the considered time horizon.
	Then, this part of the time interval starts from the beginning of the time horizon,
	and also comprises the time interval
	$[\min_{o\in O}{a_o}, \min_{o\in O}{a_o} + t_a+r^{\mathrm{des}}_{\mathrm{a}} - \max_{o\in O}{b_o}]$.
	All mobility offers whose journey intervals overlap the chosen time interval are deselected.

	\item[Demand Conflict Graph Destroy]
	The idea of this operator is to remove offers from those demands which potentially affect each other.
	For this purpose, the demand conflict graph described in Section~\ref{ssec:demandConflictGraph} is used.
	For choosing the nodes to be deselected, a start node is chosen uniformly at random.
	From this start node, a breadth-first search is started until
	the number of visited nodes is equal to the number of offers to remove $r^{\mathrm{des}}_{\mathrm{a}} = \lceil r^{\mathrm{des}}|D|\rceil$
	or all nodes of the connected component of the start node are visited.
	If all nodes in the connected component are visited but the number of offers to remove has not yet been reached,
	another start node from another connected component is chosen and the breadth-first search is started anew beginning from this node.
	This procedure is repeated until $r^{\mathrm{des}}_{\mathrm{a}}$ nodes are visited.
	Then, the selected offer of the corresponding demand of each visited node is deselected.

\end{description}
\vspace{0.1em}

\subsubsection{Repair Operators}
\label{sssec:repair}
Two different variants of repair operators are used.
\redtwo{One} is based on the ILP model
introduced in Section~\ref{ssec:ilp}
and the other on the greedy heuristic described in Section~\ref{ssec:greedy}.

\begin{description}

\item[Exact Repair]
	The exact repair operator uses the ILP model from Section~\ref{ssec:ilp} and chooses the offers to select exactly.
	Assume $O'\subset O$ are the still selected offers.
	Then, we add the constraints \mbox{$x_o = 1$ $\forall \, o \in O'$} to the model and re-solve it using the mixed integer linear programming solver.

\item[Greedy Repair]
	The greedy repair method is based on the greedy heuristic introduced in Section~\ref{ssec:greedy}. 
	By using Algorithm~\ref{alg:greedy} for the repair method as in the initial solution generation,
	the algorithm would always end up in the same solution no matter which offers are removed.
	Therefore, a randomization is introduced guided by the parameter $0<r^{\mathrm{rep}}<1$
	which determines the size of a \emph{restricted candidate list} (RCL),
	commonly used in the context of a \emph{Greedy Randomized Adaptive Search Procedure}~(\cite{Feo1995}).
	Assume we are given an ordering $(d_1,\dots,d_k)$ with $k\leq |D|$ of the demands which have not been assigned an offer yet.
	Instead of choosing the demands in exactly this order,
	the demand to consider next is chosen from the RCL,
	which is composed of the next $\lceil r^{\mathrm{rep}} |D| \rceil$ demands, uniformly at random. 
	Then, like in the greedy heuristic, the cheapest feasible offer for this demand is chosen.
	Each sort criterion for the demands from Section~\ref{ssec:greedy} can be used.

\end{description}

\section{Vehicle Classes}
\label{sec:vehicleclasses}

Often, one would expect a corporate fleet of vehicles to include many vehicles that could be used interchangeably.
This assumption allows to eliminate symmetries not exploited in the approaches presented so far in this paper.
Section~\ref{sec:definitionMOAPVC} provides an extended problem description which introduces vehicle classes.
Section~\ref{sec:ILPMOAPVC} proposes an adapted ILP model which makes use of this additional information.
Since this adapted ILP model proved to be very efficient in computational experiments (see Section~\ref{sec:results}),
no additional heuristic methods which exploit the information on vehicle classes were developed.
In Section~\ref{sec:vehicleClassDiscussion}, this modeling is compared to the MOAP of Section~\ref{sec:modeling} and
it is discussed in which cases the underlying assumption of vehicle interchangeably applies in practice.

\subsection{Modified Problem Description}
\label{sec:definitionMOAPVC}

The problem description of Section~\ref{sec:modeling} is modified as follows.
As an additional input,
we are given a set of \emph{vehicle classes}~$W$.
A given mapping $\varphi : V \rightarrow W$ assigns each vehicle~$v \in V$ to a vehicle class~\mbox{$\varphi(v) \in W$}.
All vehicles of the same class must be indistinguishable, i.e.,
there must exist an journey interval and cost preserving bijection between the offers of two arbitrary vehicles from the same class.
All mobility offers requiring a vehicle from the same class are subsumed
in an \emph{abstract mobility offer}~\mbox{$o \in \bar{O}$}, assigned to the vehicle class $w_o \in W \cup \set{\ast}$,
replacing the vehicle $v_o \in V$ assigned to each offer in the original problem definition.
As before, $w_o = \ast$~denotes that an abstract mobility offer~$o \in \bar{O}$ does not require any vehicle.
The problem is to select exactly one abstract mobility offer for each demand
such that the total cost of the selected offers is minimal, 
and to choose for each selected offer~\mbox{$o \in \bar{O}$} with $w_o \neq \ast$
a vehicle $s_o \in \varphi^{-1}(w_o)$ while overall feasibility is ensured.
Feasibility is given if for each pair~\mbox{$o, p \in \bar{O}$} of selected offers
assigned to the same vehicle $s_o = s_p \in V$,
it holds that $T_{o}\cap T_{p} =\emptyset$,
i.e., the journey intervals of all selected offers
that use the same vehicle do not overlap.
We refer to this problem as the \emph{Mobility Offer Allocation Problem with Vehicle Classes} (MOAPVC).

\subsection{Integer Linear Programming Model}
\label{sec:ILPMOAPVC}
The ILP model introduced in the following selects abstract mobility offers.
It is ensured that the number of simultaneous selections of abstract mobility offers that use the same vehicle class
does not exceed the number of vehicles available in that class.
From the problem definition, it follows directly that vehicles of the same class have identical conflict graphs.
Thus, we obtain a conflict graph~$G^w$ for each vehicle class~$w$.
Analogously to Section~\ref{ssec:ilp},
we denote the set of maximum cliques in the conflict graph~$G^w$ of a vehicle class~\mbox{$w \in W$} by $C^w$.
Decision variables $x_o \in \set{0, 1}$ determine whether an abstract mobility offer~\mbox{$o \in \bar{O}$} is selected.
The number of vehicles in a vehicle class~$w \in W$ is determined by~$\abs{\varphi^{-1}(w)}$.

\begin{align}
	  min \sum_{o \, \in \, \bar{O}} x_o \cdot c_o & & \label{ilp:obj2}\\
s.t.	\sum_{o \, \in \, \bar{O}_d} x_o \, = & ~1 & \forall \, d \in D \label{ilp:ca}\\
	\sum_{o \, \in \, K} x_o \, \leq & ~\abs{\varphi^{-1}(w)} & \forall \, w \in W, \; \forall \, K \in C^w \label{ilp:cb}\\
	x_o \, \in & ~\set{0,1} & \forall \, o\in \bar{O}
\end{align}

The objective function (\ref{ilp:obj2}) remains unchanged and minimizes the sum of the costs of chosen offers.
As before, constraints (\ref{ilp:ca}) ensure that for each mobility demand exactly one offer is chosen.
Constraints (\ref{ilp:cb}) prevent that more vehicles than available are used at the same time.

A solution of this ILP model only provides an assignment of selected offers to vehicle classes, but not to individual vehicles.
From a solution of this ILP, an assignment to individual vehicles can be computed as follows.
For each vehicle class, we consider the interval graph that contains
the journey intervals of all selected abstract mobility offers which require that vehicle class.
Assigning offers to vehicles then corresponds to finding minimum colorings
of these interval graphs with colors corresponding to vehicles.
The polynomial-time algorithm of \cite{Gupta1979} provides an efficient procedure for this task.

\redtwo{In practical scenarios and in current literature there often exists a vehicle hierarchy. 
This results in upgradable vehicles such that a vehicle of a higher class can be used instead of one of the requested class.
This feature is implicitly covered by this modeling approach by generating additional offers for each vehicle class higher than the requested class.
Additionally, the costs of these offers for higher classes should be increased to impede their selection.}

\subsection{Discussion}
\label{sec:vehicleClassDiscussion}

In application scenarios where vehicle interchangeability is given,
the ILP model proposed in this section can help to compute solutions more efficiently.
In cases where vehicle interchangeability is not given,
methods from Section~$\ref{sec:solutionAlgorithms}$ can be applied.
Note that instances of the MOAP and the MOAPVC can be directly transformed into each other
by either omitting vehicle class information or
by introducing artificial vehicle classes consisting of single vehicles.
Thus, both problem definitions are equally general.

In practice, the identification of vehicles to be used interchangeably depends
not only on distinguishing features such as the number of seats, trunk size, cost, energy consumption, or others,
but also on the possibility for users to choose mobility options based on such features.
In addition to that, vehicle specific appointments for technical inspections or maintenance prevent interchangeability.
An interesting use-case where vehicle interchangeability does not apply
are car sharing systems {where each vehicle is bound to one fixed location}.
There, although many vehicles might be of exactly the same type, most users would accept only vehicles located conveniently,
e.g., close to their home address.
This choice differs for individual users (i.e., demands), thus preventing vehicle interchangeability.
If vehicle classes are not given as an explicit input but are still comprised in many problem instances,
one could try to detect them automatically, e.g.,
following the linear time approach of \cite{lueker1979} for deciding interval graph isomorphisms.

\section{Experimental Evaluation}
\label{sec:results}

All algorithms presented in Section~\ref{sec:solutionAlgorithms}
are implemented in Java 1.8.
The general purpose mixed integer linear programming solver IBM ILOG CPLEX Optimizer, version 12.6.2, is used for solving the ILP model.
All numerical experiments were conducted using one core of
an Intel Xeon 2643 machine with 3.3 GHz and 16 GB RAM each running Linux CentOS 6.5.

\subsection{Instances}
\label{sec:instances}

\newcommand{\DU}[1]{$\sim \negthickspace DU[#1]$}

In order to evaluate the presented solution approaches, two sets of instances are generated.
The first set of instances, denoted as \emph{AG}, is created randomly and the created mobility offers have no connection to real-world data.
The second set of instances, denoted as \emph{RW}, is based, to a certain degree, on real-world statistical data and some assumptions about travel behavior.

There are several structural differences in the two instance sets which reflects the generality of the proposed model.
Instance set AG \redtwo{contains} multiple, non-overlapping time windows for mobility demands
so that multiple offers for same vehicles may have different journey intervals.
This enables modeling alternative dates for one demand.
Instance set RW, on the other hand, does not allow alternative dates
but the journey intervals of the mobility offers are considered more realistically
based on spatial, demographic, and economic data from Vienna, Austria.
The start and end time of the mobility offers consider the overall difference
in the demand over the time of the day and also differentiate between weekdays and weekends.

The data for both instance sets and the source code of the instance generators
are made publicly available (\mbox{\url{https://github.com/ait-energy/seamless}}).
This sections \redtwo{gives} an overview of the rather complex instance generation procedures.
\redtwo{We} refer to \ref{app:instance} for a more detailed description of the instance set RW. 
In the following, determining a random number according to a discrete uniform distribution over integers $[a, b]$ is denoted by \DU{a, b}.

\subsubsection{Artificially \redtwo{Generated} Instances (AG)}
First, random instances were created aiming at covering a wide range of scenarios.
This instance generation procedure reflects, to some extent,  the generality of the proposed modeling.
Most parameters of the generator are fixed in order to limit the number of instances.
Four parameters are varied which results in an overall number of 144 parameter combinations.
For each combination, one instance is randomly generated.

The instances aim at representing scenarios with a mixed fleet of vehicles and mobility demands representing a variety of situations.
A fixed \emph{number of demands}~$|D|$ to be generated
is chosen from the set $\{ 200, 1000,$ $2000, 5000 \}$.
For this number of demands and an expected duration per offer
(derivable from subsequently introduced parameters),
a \emph{fleet utilization rate}~$P_u$ is chosen from $\{ 20\%, 40\%, 60\%, 80\% \}$
and used to determine the overall number of vehicles in the fleet.
Smaller fleet utilization rates lead to more vehicles in the fleet.
Vehicles are classified into four categories
(e.g., representing small car, medium car, large car, and van);
an individual \emph{vehicle type cost factor} is assigned to each of them (2, 3, 4, and 7, respectively).
The number of vehicles per category is determined by fixed
\emph{vehicle category portions} (15\%, 35\%, 35\%, 15\%) of the overall number of vehicles.
Then, for each generated mobility demand,
a \emph{minimum vehicle category} is randomly determined according to this distribution.
Only vehicles of this or a higher category can be used to fulfill this mobility demand.
For each vehicle in a suitable category,
mobility offers are generated with a \emph{vehicle acceptance probability}~$P_a$
chosen from $\{ 0.4, 0.6, 0.8 \}$.
\red{Vehicles are chosen as suitable offers with that probability}.
Possible choices between multiple dates for the same appointment
are included by creating multiple offers with different journey intervals
for the same vehicle.
This number of created mobility offers per vehicle is determined 
by a \emph{number of journey intervals}~\DU{1, 3}.

All dates in the generated instances are represented as integers denoting hours.
Mobility demands are generated for a \emph{planning horizon}~$H$ of four weeks ($H = 24 \cdot 7 \cdot 4 = 672$).
Then, with a \emph{long demand probability} $P_l$ chosen from $\{ 0.01, 0.02, 0.05 \}$,
it is determined if the demand is considered to be \enquote{long}.
For a demand to be considered as \enquote{long} means
a \emph{base duration}~\DU{7, H} is chosen;
otherwise, a base duration~\DU{1, 6} is chosen.
This base duration of a demand predominantly determines
the duration of the journey interval of its offers.
For each mobility offer to be generated (i.e., each journey interval and vehicle),
a \emph{relative start date} is chosen from~\DU{2, 168}.
Finally, the cost of a mobility offer is determined by choosing, once for each demand,
a cost per time factor~\DU{10, 30}.
This is then used to determine the cost of each offer
by multiplying it with the duration of its journey interval
and the relative cost of the category of the used vehicle.

\begin{table}
\centering
\caption{Parameters used for generating the 144 random instances of instance set AG.}
\begin{tabular}[]{ccc}
\toprule
number of demands & \enskip $|D|$ \enskip & 200, \enskip 1000, \enskip 2000, \enskip 5000 \\
fleet utilization rate & $P_u$ & 20\%, \enskip 40\%, \enskip 60\%, \enskip 80\% \\
vehicle acceptance probability & $P_a$ & 0.4, \enskip 0.6, \enskip 0.8  \\
long demand probability & $P_l$ & 0.01, \enskip 0.02, \enskip 0.05 \\
\bottomrule
\end{tabular}
\vspace{0.5em}
\label{tab:generatorParameters}
\end{table}

\begin{table}
\centering
\caption{Instance sizes in numbers of vehicles and offers for instance set AG \\ with the parameters $P_a = 0.6$ and $P_l = 0.02$.}
\begin{tabular}[]{ccrrrr}
\toprule
& & \multicolumn{1}{c}{\enskip $|D| = 200$ \enskip} & \multicolumn{1}{c}{\enskip $|D| = 1000$ \enskip} & \multicolumn{1}{c}{\enskip $|D| = 2000$ \enskip} & \multicolumn{1}{c}{\enskip $|D| = 5000$ \enskip} \\
\midrule
\multirow{2}{*}{$P_u = 20\%$} & $|O|$ & 1578 & 21254 & 76427 & 445103 \\
& $|V|$ & 10 & 36 & 70 & 172 \\
\multirow{2}{*}{$P_u = 40\%$} & $|O|$ & 1125 & 11869 & 41265 & 229731 \\
& $|V|$ & 6 & 18 & 36 & 86 \\
\multirow{2}{*}{$P_u = 60\%$} & $|O|$ & 973 & 9036 & 29958 & 157753  \\
& $|V|$ & 4 & 12 & 24 & 58 \\
\multirow{2}{*}{$P_u = 80\%$} & $|O|$ & 977 & 7836 & 23910 & 123765 \\
& $|V|$ & 4 & 10 & 18 & 44 \\
\bottomrule
\end{tabular}
\vspace{0.5em}
\label{tab:instanceSizesV1}
\end{table}

Supplementary to offers using the considered fleet,
mobility offers representing the utilization of public transportation or taxis are included.
Thus, suitable offers which do not require any vehicle are generated.
For reflecting the usage of a taxi,
or alternatively the regret cost of not fulfilling a demand at all,
an additional mobility offer is generated \redtwo{for each mobility demand}.
Its costs are calculated based on the cost of the mobility offer (of the same demand)
with the minimum vehicle category.
The cost of this offer is set to
a \emph{taxi cost percentage}~\DU{300, 600} of the base cost.
With a \emph{public transportation probability} of $0.5$,
an additional offer is created representing a public transportation based route.
The cost of this offer is set to a \emph{public transportation cost percentage}~\DU{100, 300} of the base cost.

Table~\ref{tab:generatorParameters} provides an overview
of all varying instance generation parameters,
which yield an overall number of 144 possible combinations.
Table~\ref{tab:instanceSizesV1} shows the number of offers and vehicles
for the 16~instances generated with
a vehicle acceptance probability of $P_a = 0.6$ and
a probability for long demands $P_l = 0.02$.

\subsubsection{Instances based on Real-World Requirements (RW)}
\label{sssec:instancesV2}
The second set of instances is based on spatial, demographic, and economic data of Vienna, Austria.
The instance generation is based on a defined set of transport modes which are categorized into \emph{foot}, \emph{public transport}, \emph{bike}, \emph{battery electric vehicle (BEV)} and subtypes corresponding to specific car models, \emph{internal combustion engine vehicle (ICEV)} and subtypes corresponding to the size of the vehicle, and \emph{taxi}.
Each of these modes is defined by attributes like CO$_2$ emissions per distance, cost per time and distance, amount of additional time needed for setup (e.g., getting to the car, time needed for parking) which together defines a cost function.

\redtwo{For} each single benchmark instance an artificial company is constructed with $|P|$ employees and a number of vehicles for each transport mode with limited resources depending on $P$ and given by the instance parameter $\nu\in [0,1]$.
For each transport type with limited availability such as cars and bikes, there are \DU{0,\lfloor\nu P\rfloor} such vehicles available.
For each person $p\in P$, a typical work week with work-related and private events is constructed which forms the set of mobility demands of $p$.
\redtwo{Based} on the personal preferences of~$p$ regarding mode of transport (which depend on statistical data, e.g., regarding gender and probability of owning a driving license), for each acceptable transport mode, one mobility offer is created for each resource of that mode with the corresponding journey interval and cost.
The journey interval depends on the time given by the mobility demand and the travel and setup time of the corresponding mode of transport.
The cost consists of three factors:
the cost given by the distance,
the cost given by CO$_2$ emissions, and
the cost given by the time.
While the \red{costs for distance and CO$_2$} are fixed for a specific mobility demand and mode of transport,
the \red{costs for time} are only considered for business mobility demands and based on average salaries.

For each combination of company size $|P|\in \{500, 750, 1000, 1250, 1500, 1750\}$ and
relative number of vehicles $\nu\in \{0.05, 0.10, 0.15\}$ 30 instances are created. 
\redtwo{Table~\ref{tab:instanceSizesV2} provides} an overview of the average size of the instances regarding
the number of offers and number of vehicles.
\ref{app:instance} provides a more detailed description of this instance generation procedure.

\begin{table}[htbp]
  \centering
  \caption{Instance sizes in average numbers of vehicles and offers for instance set RW.}
    \begin{tabular}{c|lrrr||c|lrrr}
          &       & $\nu = $ 0.05  & $\nu = $ 0.1   & $\nu = $ 0.15  &       &       & $\nu = $ 0.05  & $\nu = $ 0.1   & $\nu = $ 0.15 \\
    \midrule
    \multirow{2}[1]{*}{$|P| = 500$} & $|O|$   & 83783.6 & 175242.2 & 264657.2 & \multirow{2}[1]{*}{$|P|$ = 1250} & $|O|$   & 526631.2 & 980385.6 & 1462976.7 \\
          & $|V|$  & 71.1  & 152.0 & 225.3 &       & $|V|$   & 190.8 & 350.2 & 566.7 \\
    \multirow{2}[0]{*}{$|P|$ = 750} & $|O|$   & 177830.6 & 368634.7 & 537975.3 & \multirow{2}[0]{*}{$|P|$ = 1500} & $|O|$   & 747897.3 & 1580761.5 & 2245212.8 \\
          & $|V|$  & 100.8 & 209.1 & 332.9 &       & $|V|$  & 227.1 & 472.6 & 682.9 \\
    \multirow{2}[0]{*}{$|P|$ = 1000} & $|O|$   & 349469.7 & 643205.2 & 942703.7 & \multirow{2}[0]{*}{$|P|$ = 1750} & $|O|$  & 940197.0 & 1936257.1 & 3085115.9 \\
          & $|V|$   & 155.7 & 293.5 & 427.8 &       & $|V|$  & 249.8 & 538.4 & 820.4 \\
    \end{tabular}%
  \label{tab:instanceSizesV2}%
\end{table}%

\let\DU\undefined

\subsection{Computational Results of the ILP Model}

\begin{table}[bt]
  \centering
  \caption{Computational results of the exact solution approach using the ILP model for instance set AG.} \medskip
    \begin{tabular}{rr|rrr||rr|rrr}
    \multicolumn{5}{c}{Aggregated by $|D|$ }      & \multicolumn{5}{c}{Aggregated by $P_l$}  \\ \midrule
    $|D|$ & \#I & \#S & $\overline{\mathrm{gap}}$ & $\overline{t}$[s] & $P_l$ & \#I  & \#S & $\overline{\mathrm{gap}}$ & $\overline{t}$[s] \\ \midrule   
   
    200   & 36 & 36 & 0.00\% & $<1$     & 0.01  & 48 & 36 & 6.94\% & 1054 \\
    1000  & 36 & 36 & 0.00\% & 100   & 0.02  & 48 & 32 & 11.13\% & 1295 \\
    2000  & 36 & 22 & 0.13\% & 1652  & 0.05  & 48 & 30 & 20.88\% & 1465 \\
    5000  & 36  & 4 & 51.81\% & 3333  &       &       &       &  \\ \midrule \midrule

		\multicolumn{5}{c}{Aggregated by $P_u$ }      & \multicolumn{5}{c}{Aggregated by $P_a$ }  \\ \midrule
		$P_u$ & \#I & \#S & $\overline{\mathrm{gap}}$ & $\overline{t}$[s] & $P_a$ & \#I & \#S & $\overline{\mathrm{gap}}$ & $\overline{t}$[s] \\ \midrule   
    
    20\%    & 36 & 31 & 2.78\% & 802   & 0.4    & 48 & 33 & 8.91\% & 1240 \\
    40\%    & 36 & 21 & 22.54\% & 1648  & 0.6    & 48 & 31 & 13.05\% & 1329 \\
    60\%    & 36 & 22 & 17.62\% & 1420  & 0.8    & 48 & 34 & 17.00\% & 1244 \\
    80\%    & 36 & 24 & 9.00\% & 1215  &       &       &       &  \\

    \end{tabular}%
  \label{tab:exactResultsV1}%
\end{table}%

\begin{table}[htbp]
  \centering
  \caption{Computational results of the exact solution approach using the ILP model for instance set RW.}\medskip
    \begin{tabular}{rr|rrcrrr|rrcrrr}
    \multicolumn{1}{l}{$\nu$} & \multicolumn{1}{l}{\#I} & \multicolumn{1}{l}{$|P|$} & \multicolumn{1}{l}{\#S} & \multicolumn{1}{l}{\#N/S} & \multicolumn{1}{l}{$\overline{\mathrm{gap}}$} & \multicolumn{1}{l}{$\overline{t}$[s]} & \multicolumn{1}{l|}{$\overline{t^*}$[s]} & \multicolumn{1}{l}{$|P|$} & \multicolumn{1}{l}{\#S} & \multicolumn{1}{l}{\#N/S} & \multicolumn{1}{l}{$\overline{\mathrm{gap}}$} & \multicolumn{1}{l}{$\overline{t}$[s]} & \multicolumn{1}{l}{$\overline{t^*}$[s]} \\
   \midrule
    0.05  & 30    &       & 30    & -     & 0.0\% & 23    & 1     &       & 26    & \multicolumn{1}{c}{-} & 0.6\% & 1593  & 3 \\
    0.1   & 30    & 500   & 30    & -     & 0.0\% & 55    & 1     & 1250  & 15    & \multicolumn{1}{c}{-} & 38.7\% & 2575  & 3 \\
    0.15  & 30    &       & 30    & -     & 0.0\% & 70    & 1     &       & 15    & \multicolumn{1}{c}{-} & 51.9\% & 2639  & 3 \\
    \midrule
    0.05  & 30    &       & 30    & -     & 0.0\% & 131   & 1     &       & 18    & \multicolumn{1}{c}{-} & 43.3\% & 2445  & 3 \\
    0.1   & 30    & 750   & 30    & -     & 0.0\% & 330   & 1     & 1500  & 7     & \multicolumn{1}{c}{-} & 60.6\% & 3350  & 4 \\
    0.15  & 30    &       & 30    & -     & 0.0\% & 299   & 1     &       & 8     & \multicolumn{1}{c}{-} & 75.5\% & 3150  & 4 \\
    \midrule
    0.05  & 30    &       & 30    & -     & 0.0\% & 651   & 2     &       & 10    & \multicolumn{1}{c}{-} & 45.7\% & 3060  & 4 \\
    0.1   & 30    & 1000  & 26    & -     & 0.4\% & 1238  & 2     & 1750  & 5     & 2     & 57.3\% & 3352  & 5 \\
    0.15  & 30    &       & 28    & -     & 42.1\% & 1728  & 2     &       & 2     & 6     & 75.4\% & 3395  & 5 \\

    \end{tabular}%
  \label{tab:exactResultsV2}%
\end{table}%

Extensive computational experiments using \redtwo{both instance sets} were performed.
For determining CPLEX parameters, we used the built-in parameter tuning tool on the
set of training instances also used for ALNS parameter tuning (see Section~\ref{sec:ALNSresults}) with a time limit of one
hour. It turned out that the default parameters of CPLEX worked best.
Tables~\ref{tab:exactResultsV1} and~\ref{tab:exactResultsV2} show results obtained by
the exact solution approach using the ILP model from Section~\ref{ssec:ilp} for instance set AG and RW, respectively.
In Table~\ref{tab:exactResultsV1} results are aggregated over the instance generation parameters $|D|$, $P_u$, $P_a$, and $P_l$ and in Table~\ref{tab:exactResultsV2} the results are grouped by the instance parameters $|P|$ and $\nu$.
A time limit of at most one hour per instance was used.
The number of instances per group is given in columns~\enquote{\#I}.
\redtwo{C}olumns~\enquote{\#S} provide the number of solved instances.
\redtwo{C}olumns~\enquote{$\overline{\mathrm{gap}}$} provide the average optimality gap between the objective function value of the found solution and the lower bound in percent.
\redtwo{C}olumns~\enquote{\#N/S} provide the number of instances in which no feasible solution was found, either because of the time or the memory limit.
\redtwo{C}olumns~\enquote{$\overline{t}$[s]} provide the average runtime in seconds over all instances that were solved to optimality.
For the instance set RW, columns~\enquote{$\overline{t^*}$[s]} show the average time needed to solve the instances with the ILP model utilizing vehicle classes as described in Section~\ref{sec:ILPMOAPVC}.
\redtwo{Note that the optimal solution value for both variants with and without vehicle classes remains equal and that the time needed for computing actual vehicle assignments based on the vehicle class assignments is negligible.}

In instance set AG, all small and medium instances ($|D| \leq 2000$)
are either solved to optimality or the obtained solution shows a very small gap.
Only four large instances ($|D| = 5000$) are solved to optimality.
The large gaps indicate that the solution found for the larger instances are quite poor.
\redtwo{For} $P_u = 20~\%$ and for $P_u = 80~\%$, more instances are solved to optimality or only a small gap is obtained.
Instances with lower vehicle utilization rates ($P_u = 20~\%$)
\redtwo{are easier to solve because} they have less conflicts between offers.
Instances with higher vehicle utilization rates ($P_u = 80~\%$)
\redtwo{have fewer} feasible \redtwo{solutions} which makes \redtwo{them} easier to solve as well.
The impact of the parameters~$P_l$ and~$P_a$ is less clear.
\redtwo{Long mobility demands increase the number of conflicts}
which explains that instances with large values of~$P_l$ \redtwo{are harder to solve}.
Instances with lower values for~$P_a$ seem to be easier,
which might \redtwo{result from} fewer decision variables.

In instance set RW, CPLEX was able to solve all instances with $|P|\in \{500,750\}$ to optimality
within the time limit \redtwo{of one hour}. 
The larger instances with $|P|\geq 1250$ are harder to solve.
\redtwo{The} difficulty increases with $|P|$ but also with $\nu$.
The large optimality gaps for these instances and the high \redtwo{runtimes} show the necessity of a fast and efficient heuristic algorithm.
For the largest instance set with $|P|=1750$ and $\nu\geq 0.10$,
CPLEX could not \redtwo{even} find any feasible solution in 8~cases. 
When \redtwo{using} vehicle classes \redtwo{instead},
all instances can be solved within a few seconds.
As described before, this model has much less variables and constraints and can therefore be solved faster.
Interchangeable vehicles only exist for instances RW, therefore we only report the results for this instance class.

\subsection{Computational Results of the ALNS}
\label{sec:ALNSresults}

For evaluating the (A)LNS and its operators proposed in Section~\ref{ssec:LNS} we first have to select appropriate parameters.
We choose them by using the parameter tuning tool \emph{irace}~\cite{lopez2016irace} which iteratively samples the parameter space, evaluates the samples, and discards them if a Friedman test shows that it is dominated by other parameter configurations. 
Therefore, we generate a new training set of instances based on instance set RW described in Section~\ref{sssec:instancesV2} with $|P|=500$ and $\nu\in\{0.05,0.10,0.15\}$.
For each of these 3 combinations, 10 instances are created resulting in a training set consisting of 30 instances.
On these instances, irace is executed with a total of 20000 runs and a time limit of 1 minute per run.
To get more detailed data about the performance of the different LNS configurations,
one run of irace is performed for each destroy operator (\emph{Random}, \emph{Time Interval}, and \emph{Demand Conflict Graph}).
\redtwo{Based} on these results, one run of irace is performed for the ALNS utilizing all three destroy operators each with $r^\mathrm{des}=0.15$ and with the two repair operators, \emph{exact} repair and \emph{greedy} repair with sorting criterion \emph{MaxMinCost} and $r^\mathrm{rep}=0.1$.
The parameter space in which irace operates is shown in Table~\ref{tab:paramTuning}.

\begin{table}
\centering
\caption{Parameter tuning scenario.}
\begin{tabular}{p{20em}|p{20em}}
Description & Possible Values \\ \midrule
Repair method & \{\textit{greedy, exact}\} \\
Sort criterion for greedy repair & \{\textit{MinMinCost, MaxMinCost, \newline  MinMinCostPerTime, MaxMinCostPerTime, \newline MinAveCost, MaxAveCost} \} \\
The relative amount of solution parts that are destroyed $r^\mathrm{des}$ & [0.1,0.9] \\
Relative size of the RCL $r^\mathrm{rep}$ for the greedy repair & [0.1,0.9] \\
$\sigma_1, \sigma_2, \sigma_3$ & [1,100] \\
Decay parameter $\lambda$ & [0.01,0.99] \\
Initial temperature control parameter $\omega$ & [0.00001, 0.1]  \\
Initial temperature acceptance parameter $p^\omega$ & [0.01,0.99] \\
Cooling rate $c$ & [0.1,0.999999]
\end{tabular}
\label{tab:paramTuning}
\end{table}

The resulting parameter configurations are shown in Tables~\ref{tab:paramTuningResults} and~\ref{tab:alnsParamTuningResults}.
Note that the $\sigma$-values may seem counter-intuitive at first glance because $\sigma_1<\sigma_2<\sigma_3$.
However, this can be explained by the following observations:
First, a higher $\sigma_2$ and $\sigma_3$ value encourages diversification especially at the beginning of the algorithm because then it is more likely to accept a non-improving solution candidate. 
Second, for the exact repair operator the value of $\sigma_3$ is irrelevant because the generated solution candidate cannot be worse than the solution candidate before the destruction.
Therefore, a higher $\sigma_3$ value encourages the use of the greedy repair method which is more likely to improve the solution at the beginning of the algorithm.
In the remaining section, we refer to these values when mentioning the \emph{Random}, \emph{Time Interval}, \emph{Demand Conflict Graph}, or \emph{ALNS} configuration.

\begin{table}
\centering
\caption{LNS Parameter tuning results.}
\begin{tabular}{l|llll}
Destroy method (fixed) & Repair method & Sort criterion & $r^\mathrm{rep}$ & $r^\mathrm{des}$\\ \midrule
Random & \emph{greedy} & \emph{MaxMinCost} & 0.1308 & 0.1580 \\
Time Interval & \emph{greedy} & \emph{MaxMinCost} & 0.1063 & 0.1053\\
Demand Conflict Graph & \emph{exact} & - & - & 0.1836 \\
\end{tabular}
\label{tab:paramTuningResults}
\end{table}

\begin{table}
\centering
\caption{ALNS Parameter tuning results.}
\begin{tabular}{l|lllllll}
Method & $\sigma_1$ & $\sigma_2$ & $\sigma_3$ & $\lambda$ & $\omega$ & $p^\omega$ & $c$ \\ \midrule
ALNS & 23 & 40 & 50 & 0.2377 & 0.0373 & 0.656 & 0.2267
\end{tabular}
\label{tab:alnsParamTuningResults}
\end{table}

As expected, the performance of the greedy algorithms strongly depends on the sorting criterion that is used.
Figure~\ref{fig:greedyV1} compares the greedy algorithms on instance set AG normalized by relative differences to the best solution obtained by all algorithms for the corresponding instance.
Additionally, we compare our proposed greedy algorithms to the greedy algorithm \emph{G1-MW} proposed in~\cite{ng2014graph},
where \emph{G1-MW} is the best performing heuristic
evaluated in that paper (see Table~2, \cite{ng2014graph}).
The greedy algorithm \emph{G1-MW} always chooses the cheapest feasible offer from all remaining, not yet satisfied demands.
\emph{MaxMinCost} outperforms all other sorting criteria and yields a starting solution of about $10\%$ worse than the best found solution on average.
\emph{MinMinCost}, \emph{MinMinCostPerTime}, \emph{MinAveCost}, and \emph{Random} perform poorly and are omitted for readability.
\emph{MaxAveCost} and \emph{MaxMinCostPerTime} cannot compete with \emph{MaxMinCost}.
Algorithm \emph{G1-MW} also does not perform well.
We assume this is mainly because it is too greedy in the sense that it does not consider the consequences of always choosing the currently cheapest offer.
The computational time per instance for all variants of the greedy algorithms is always below one second per instance.

\begin{figure}[htb]
	\centering
    \includegraphics[width=0.9\textwidth]{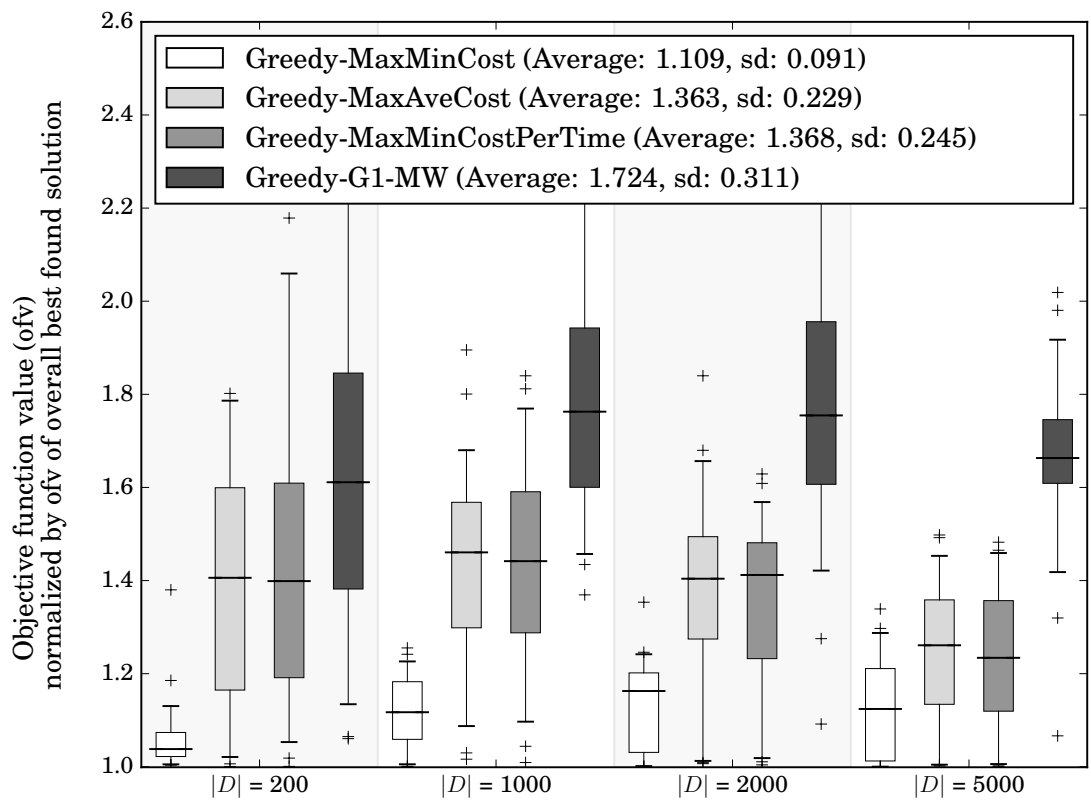}
    \caption{Comparison of different greedy construction heuristics for instance set AG.}
    \label{fig:greedyV1}
\end{figure}

To compare all proposed approaches at a glance, Figures~\ref{fig:lnsV1} and~\ref{fig:lnsV2} show a comparison of the (A)LNS configurations \emph{Random}, \emph{Time Interval}, \emph{Demand Conflict Graph}, and \emph{ALNS} along with the results obtained by CPLEX for instance sets AG and RW, respectively.
All runs of the (A)LNS in the following tests are executed with a time limit of 5 minutes each.
Regarding the results of instance set AG, for most of the small and medium instances CPLEX finds exact solutions, but it does not perform well on the large instances.
While on the smallest instances with $|D|=200$ most configurations perform more or less equally good, CPLEX outperforms the other approaches for the medium instances.
The largest instances can be solved best by the LNS configuration \emph{Demand Conflict Graph} and the \emph{ALNS} which both are on average within 1.5-1.8\% of the best found solution. 
Note that the bad results by CPLEX on the larger instances come from the time limit of 1 hour, after which the execution is halted and the best feasible solution is reported.
\emph{Random} and \emph{Time Interval} have similar performance although \emph{Random} seems to give better results on average by about 2.1\%.
Overall, depending on the available computational time and the instance to be solved, either CPLEX or the \emph{ALNS} is the best suited choice.

\begin{figure}[htb]
	\centering
    \includegraphics[width=0.9\textwidth]{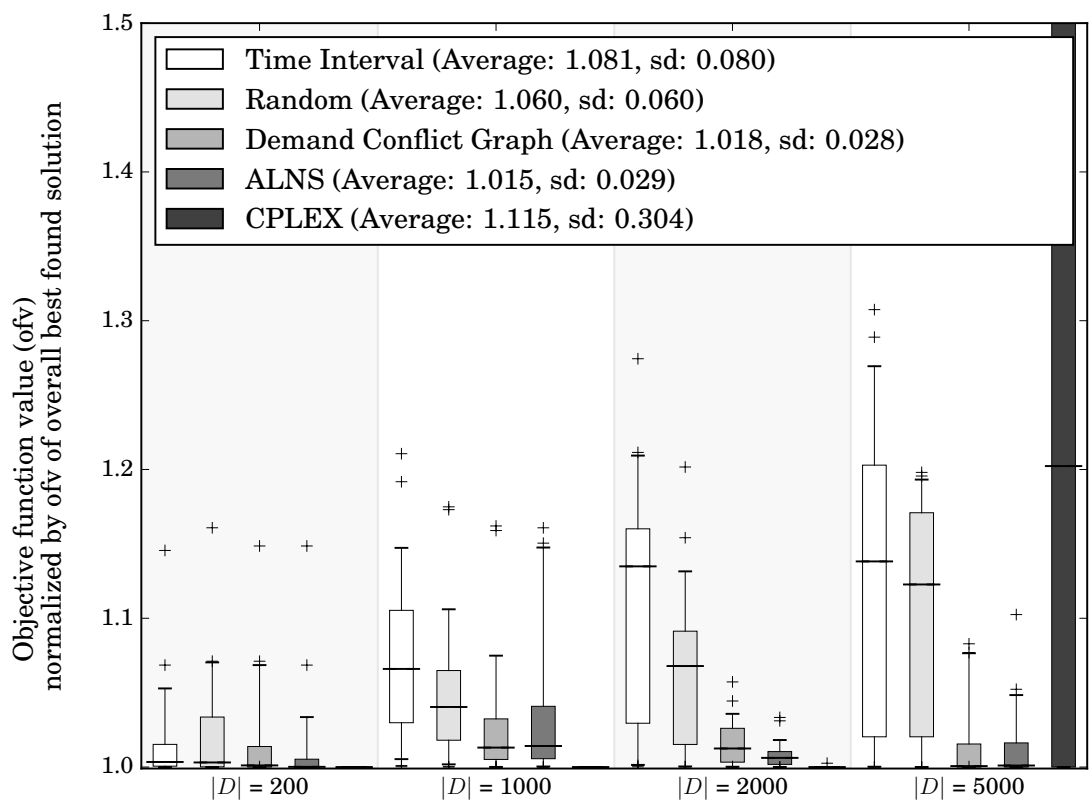}
    \caption{Comparison of the different destroy operators of the LNS, the ALNS, and CPLEX for instance set AG.}
    \label{fig:lnsV1}
\end{figure}

Figure~\ref{fig:lnsV2} shows the results on instance set RW in which
the numerical values are aggregated over the instances with
the same value for $|P|$ and $\nu$ by computing the mean of the relative differences to
the best found objective values.
These results confirm the conclusions from the results of instance set AG to some extent.
\redtwo{In} these instances, CPLEX as well is able to solve the small to medium instances and the ALNS seems to be the best choice for larger instances when comparing the average value over all instances.
For the larger instances, CPLEX and also the LNS configuration \emph{Demand Conflict Graph}
are in some cases not able to find feasible solutions. 
Therefore, we conclude that for smaller instances with $|P| \leq 1000$ the CPLEX model
can be used if the run-time is not that important.
On the other hand, for larger instances and when short run-times are required,
the ALNS is recommended.
Note that although the relative differences are much smaller
in this instance set compared to the previous,
the absolute values of the objective function are much larger due to the structure of the instances.
To summarize, Table~\ref{tab:resultApproachSummary} \redtwo{recommends methods}
for different instance sizes based on the number of employees or the number of demands.

\begin{figure}[htb]
	\centering
    \includegraphics[width=0.9\textwidth]{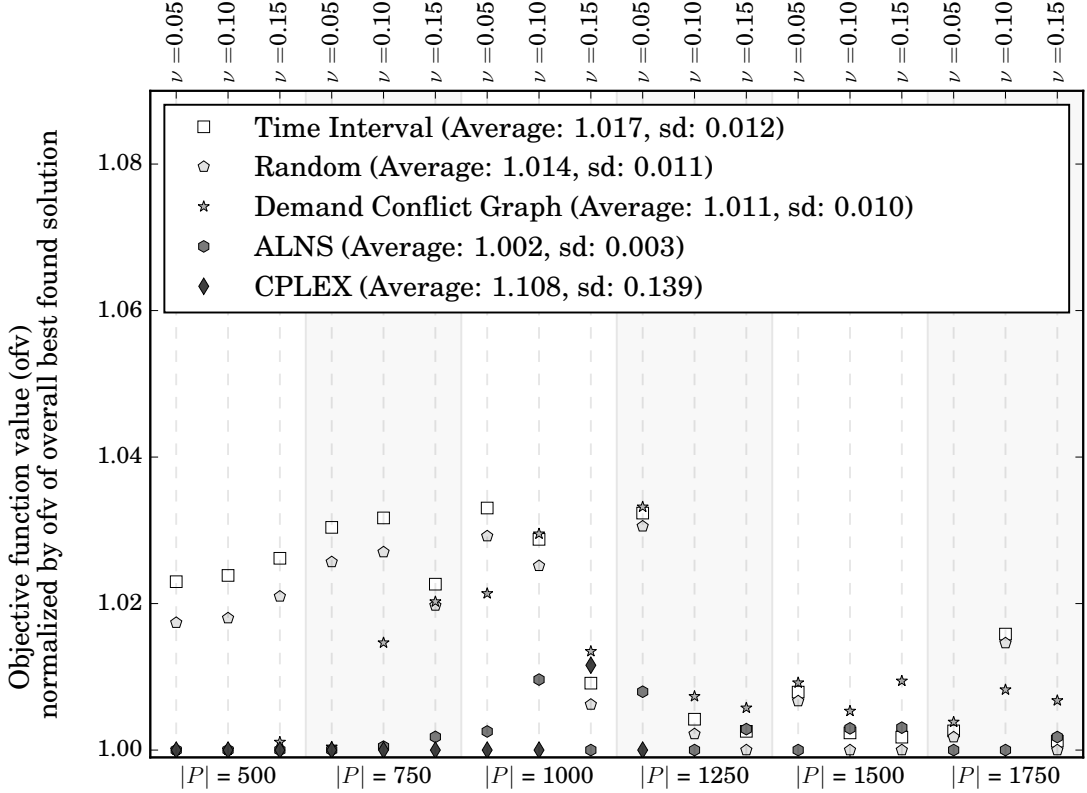}
    \caption{Comparison of the ALNS, the different destroy operators of the LNS, and the results obtained by CPLEX for instance set RW, normalized by best found solution. Additionally, the results are aggregated over the 30 instances per set and sorted by $\nu$ within their category regarding $|P|$ in ascending order.}
    \label{fig:lnsV2}
\end{figure}

\begin{table}
\centering
\caption{\red{Method recommendations for different instance sizes.}}
\red{\begin{tabular}{l|ll}
~ & fast runtime required ($<1$s) & fast runtime not required \\ \midrule
$|P|\leq 1000$ or $|D|\leq 2000$ & \emph{greedy-MaxMinCost} & \emph{CPLEX}  \\
$|P|>1000$ or $|D|>2000$ & \emph{greedy-MaxMinCost} & \emph{ALNS} \\
\end{tabular}}
\label{tab:resultApproachSummary}
\end{table}

\redtwo{To assess} the convergence behavior of the ALNS, Figure~\ref{fig:convergence} shows the decrease over time of the objective function value for 30 independent runs of the ALNS on one instance.
\redtwo{Therefore}, the currently best objective values of each run are collected in discrete time intervals of 10 seconds over the whole run-time of one hour.
We observe a much higher variance between runs during the initial phase of the search.
\redtwo{Even} the worst runs after about 1000~seconds outperform the best runs before 200~seconds.
Thus, performing multiple independent runs with the goal of obtaining a better overall result
can be recommended only if the available time per run is sufficiently high.
\red{Additionally, we evaluated the time needed to find the best solution within a runtime of 1000 seconds for the instances RW.
The results showed that for the smaller (E500, E750) and larger (E1500, E1750) instances the best solution was found in only about 44.81\% of the runtime and for the medium (E1000, E1250) instances in about 68.08\% of the runtime.
This indicates that for these instances the ALNS is usually able to converge to its final value in between 7 and 11 minutes of runtime.}

\begin{figure}[htb]
	\centering
    \includegraphics[width=0.9\textwidth]{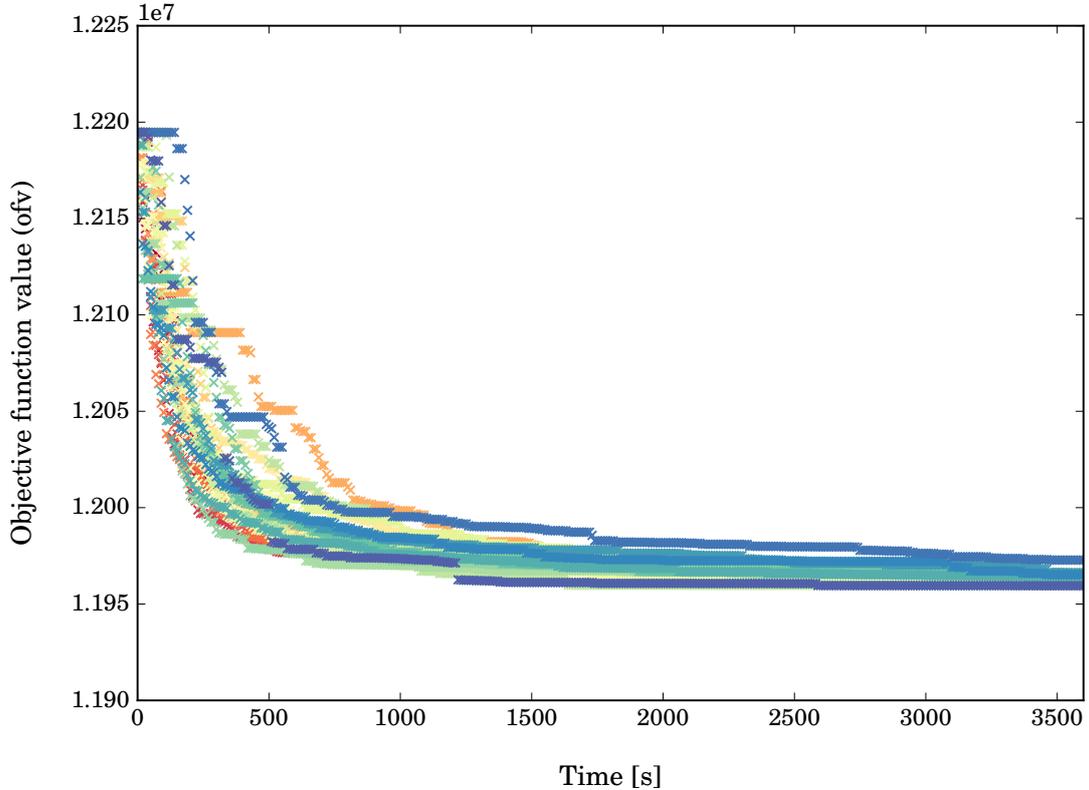}
    \caption{Typical progression of the objective value over time for the ALNS in 30 independent runs of \\ the instance \emph{E1250\_10} with $\nu=0.10$.}
    \label{fig:convergence}
\end{figure}

\red{
Finally, we analyze the properties of the results regarding their application to practical use cases.
Therefore, we computed the average utilization of the shared fleet vehicles $P_u^*$ and the relative number of trips using one of these vehicles $V^u$.
For the average utilization\redtwo{,} we computed the total reservation times of the vehicles and compared these to the whole time horizon of the specific instance.
We show these values for instance set AG \redtwo{in Table~\ref{tab:analysisV1}} and
for instance set RW \redtwo{in Table~\ref{tab:analysisV2}}. 
In the latter table we additionally compare the changes of the traditional fleet size to the shared fleet size $\Delta V$ and the change in the vehicle utilization rates $\Delta P_u^*$.
To compute these values we assumed a classical one-to-one assignment of vehicles to employees, i.e., in the instances with $|P|$ we assumed a fleet size of 500.
}

\begin{table}[bt]
  \centering
  \caption{\red{Fleet size, vehicle utilization, and data about the number of trips using shared vehicles for instance set AG.}} \medskip
    \red{\begin{tabular}{r|rrr||r|rrr}
    \multicolumn{4}{c}{Aggregated by $|D|$ }     & \multicolumn{4}{c}{Aggregated by $P_l$}  \\ \midrule
    $|D|$ & $|V|$ & $P_u^*$ & $V^u$ & $P_l$ & $|V|$ & $P_u^*$ & $V^u$  \\ \midrule     
    200   & 6.7 & 36.53\% &  66.69\%  & 0.01  & 25.9 & 47.60\% & 65.81\%  \\
    1000  & 32.0 & 49.06\% & 72.06\% &0.02  & 38.0 & 47.79\% & 72.19\%  \\
    2000  & 44.7 & 50.19\% & 75.58\% & 0.05  & 73.9 & 45.08\% & 81.19\% \\
    5000  & 109.3 & 51.53\% & 77.92\%  &       &       &       \\ \midrule \midrule
		\multicolumn{4}{c}{Aggregated by $P_u$ }      & \multicolumn{4}{c}{Aggregated by $P_a$}  \\ \midrule
		$P_u$ & $|V|$ & $P_u^*$ & $V^u$ & $P_a$ & $|V|$ & $P_u^*$ & $V^u$ \\ \midrule    
    20\%    & 86.7 & 30.64\% & 92.06\%   & 0.4    & 45.9 & 45.48\% & 70.35\% \\
    40\%    & 44.2 & 47.06\% & 79.19\%  & 0.6    & 45.9 & 46.98\% & 73.48\%  \\
    60\%    & 30.0 & 52.58\% & 66.44\%  & 0.8    & 45.9 & 48.02\% & 75.35\% \\
    80\%    & 22.8 & 57.03\% & 54.56\%  &      &       &       &  \\
    \end{tabular}}%
  \label{tab:analysisV1}%
\end{table}%

\red{
The results of instance set AG show a generally higher utilization rate than
in instance set RW and increasing with higher values of $|D|$ and $|P_u|$.
Since the instance generation procedure of the instance set AG does not distinguish between day and night trips, the demand is more evenly distributed which results in the increase of the utilization rate.
Also, the utilization increases with a higher demand to vehicle ratio and clearly with an increasing input vehicle utilization ratio.
Considering the results of instance set RW, we see a reduction of the fleet size between 50 and 85\% while the average utilization rates are increasing in most cases.
There are, however, some cases in which the vehicle utilization rates are actually lower than in the case of fixed vehicle assignments.
This is because the goal of the algorithm is not to maximize the vehicle usage but to find the most cost-effective mobility offer allocations.
This means that in many cases the alternative offers, e.g., public transport, \redtwo{are} actually better for the company and are therefore chosen over a shared fleet vehicle. 
The same effect can also be observed when looking at the number of offers using fleet vehicles which are also rather low.
If desirable, the objective function of the optimization problem could be replaced by a utility maximization function in order to increase the utilization values $P_u^*$ and $V^u$.
In this work, however, we did not pursue this variant because we believe in the benefit of offering alternative mobility offers to employees.
}

\begin{table}[htbp]
  \centering
  \caption{\red{Fleet size, vehicle utilization, and data about the number of trips using shared vehicles for instance set RW.}}\medskip
    \red{\begin{tabular}{r|crrrr|crrrr}
    \multicolumn{1}{l}{$\nu$} & \multicolumn{1}{l}{$|P|$} & $\Delta V$ & $P_u^*$ & $\Delta P_u^*$ & $V^u$ & \multicolumn{1}{l}{$|P|$} & $\Delta V$ & $P_u^*$ & $\Delta P_u^*$ & $V^u$ \\
   \midrule
    0.05  &       & -85.78\% & 59.33\% & 8.51\% & 22.93\% &       & -84.74\% & 58.50\% & 8.29\% & 23.27\% \\
    0.1   & 500   & -69.60\% & 51.80\% & 0.97\% & 31.67\% & 1250  & -71.98\% & 53.87\% & 3.66\% & 30.17\% \\
    0.15  &       & -54.94\% & 43.87\% & -6.96\% & 21.90\% &       & -54.66\% & 52.66\% & 2.45\% & 31.10\% \\
    \midrule
    0.05  &       & -86.56\% & 60.17\% & 9.45\% & 22.23\% &       & -84.86\% & 58.70\% & 8.42\% & 23.63\% \\
    0.1   & 750   & -72.12\% & 53.47\% & 2.75\% & 35.23\% & 1500  & -68.49\% & 54.89\% & 4.62\% & 30.07\% \\
    0.15  &       & -55.61\% & 45.20\% & -5.52\% & 32.50\% &       & -54.47\% & 52.17\% & 1.89\% & 26.33\%  \\
    \midrule
    0.05  &       & -84.43\% & 57.80\% & 7.47\% & 24.57\% &       & -85.73\% & 58.50\% & 8.22\% & 21.83\% \\
    0.1   & 1000  & -70.65\% & 52.13\% & 1.81\% & 31.53\% & 1750  & -69.23\% & 52.80\% & 2.62\% & 29.53\% \\
    0.15  &       & -57.22\% & 49.83\% & -0.49\% & 33.23\% &       & -53.12\% & 53.63\% & 3.44\% & 27.38\% \\

    \end{tabular}}%
  \label{tab:analysisV2}%
\end{table}%

\section{Conclusion and Outlook}
\label{sec:conclusions}
This paper \redtwo{introduces} the \emph{Mobility Offer Allocation Problem}
\redtwo{for corporate mobility services} and solution algorithms \redtwo{to solve it}.
\redtwo{We propose} a methodology
that integrates a mixed fleet of vehicles with other mobility options
such as public transportation or taxis.
An experimental evaluation shows the trade-offs of the proposed solution methods
regarding computational times and solution quality for different kinds of instances.
The results demonstrate the applicability of the methods for realistic instance sizes \red{and show performance indicators interesting for real world applications}.
For improving the proposed adaptive large neighborhood search,
designing better destroy operators seems worth investigating.
Currently, when choosing demands to be destroyed,
the selected offers are not taken into account. 
Including the potential cost savings for a demand
might lead to more efficient operators.
In practice, such approaches must be applied in a dynamic setting
where demands arrive over time and offers are booked in advance.
There, the idea of delaying assignment decisions in order to increase planning flexibility
provides further practical benefits, e.g., in case of vehicle breakdowns or late returns.
We believe the proposed approaches are applicable in such rolling horizon scenarios.
\redtwo{However}, further evaluating and adapting them in such settings
is a relevant direction of future research.

The conflict graph based modeling of the problem facilitates the inclusion of additional constraints.
For example, a consideration of persons, potentially involved in multiple appointments,
can be included via additional conflict edges.
\red{Another possible extension would be to suggest multiple employees to share a
vehicle if their requests are similar, e.g.,~\cite{enzi2020modeling}. This could be modeled by
introducing offers that satisfy more than one demand.}
A major limitation of the proposed modeling is that
mobility offers refer to fixed journey intervals.
In case the sequencing of offers becomes relevant,
one would obtain a variant of the well-known Vehicle Routing Problem.
Though journey intervals cannot, locations actually can vary in the presented modeling.
Assuming a future scenario with a fleet of self-driving vehicles,
each mobility demand could state a fixed start and end location
specifying a request for a ride between given locations taking a fixed amount of time.
Then, two offers using the same vehicle are in conflict
if driving from the end location of the earlier offer
to the start location of the later offer is not possible
within the given time.
So, an adapted conflict graph based modeling could prove useful also for such scenarios.
The approach also aims at fostering the use of battery electric vehicles
by helping to achieve utilization rates required for compensating high purchase prices.
Recharging processes can either be included simplistically,
by prolonging the journey intervals of mobility offers,
or, a more detailed modeling could combine ideas from this work and that of \cite{Sassi2014},
where battery loading states are modeled explicitly.

Overall, we believe the proposed modeling provides a flexibility
that offers a range of interesting applications not restricted to corporate environments,
e.g., a large housing unit equipped with a fleet of cars shared by the inhabitants.
A~larger scale application would be \redtwo{to implement the}
approach for station based car-sharing providers.

\bibstyle{plain}
\bibliography{literature}

\begin{thebibliography}{47}
\expandafter\ifx\csname natexlab\endcsname\relax\def\natexlab#1{#1}\fi
\providecommand{\url}[1]{\texttt{#1}}
\providecommand{\href}[2]{#2}
\providecommand{\path}[1]{#1}
\providecommand{\DOIprefix}{doi:}
\providecommand{\ArXivprefix}{arXiv:}
\providecommand{\URLprefix}{URL: }
\providecommand{\Pubmedprefix}{pmid:}
\providecommand{\doi}[1]{\href{http://dx.doi.org/#1}{\path{#1}}}
\providecommand{\Pubmed}[1]{\href{pmid:#1}{\path{#1}}}
\providecommand{\bibinfo}[2]{#2}
\ifx\xfnm\relax \def\xfnm[#1]{\unskip,\space#1}\fi
\bibitem[{Bates and Leibling(2012)}]{bates2012spaced}
\bibinfo{author}{J.~Bates}, \bibinfo{author}{D.~Leibling},
  \bibinfo{title}{Spaced Out: Perspectives on parking policy},
  \bibinfo{type}{Report}, RAC Foundation, \bibinfo{year}{2012}.
\bibitem[{Shoup(2017)}]{shoup2017high}
\bibinfo{author}{D.~Shoup}, \bibinfo{title}{The High Cost of Free Parking:
  Updated Edition}, \bibinfo{publisher}{Routledge}, \bibinfo{year}{2017}.
\bibitem[{Oliveira et~al.(2017)Oliveira, Carravilla, and
  Oliveira}]{oliveira2017fleet}
\bibinfo{author}{B.~B. Oliveira}, \bibinfo{author}{M.~A. Carravilla},
  \bibinfo{author}{J.~F. Oliveira},
\newblock \bibinfo{title}{Fleet and revenue management in car rental companies:
  A literature review and an integrated conceptual framework},
\newblock \bibinfo{journal}{Omega} \bibinfo{volume}{71} (\bibinfo{year}{2017})
  \bibinfo{pages}{11--26}.
\bibitem[{Hoff et~al.(2010)Hoff, Andersson, Christiansen, Hasle, and
  L{\o}kketangen}]{hoff2010industrial}
\bibinfo{author}{A.~Hoff}, \bibinfo{author}{H.~Andersson},
  \bibinfo{author}{M.~Christiansen}, \bibinfo{author}{G.~Hasle},
  \bibinfo{author}{A.~L{\o}kketangen},
\newblock \bibinfo{title}{Industrial aspects and literature survey: Fleet
  composition and routing},
\newblock \bibinfo{journal}{Computers \& Operations Research}
  \bibinfo{volume}{37} (\bibinfo{year}{2010}) \bibinfo{pages}{2041--2061}.
\bibitem[{Ma et~al.(2012)Ma, Balthasar, Tait, Riera-Palou, and
  Harrison}]{ma2012new}
\bibinfo{author}{H.~Ma}, \bibinfo{author}{F.~Balthasar},
  \bibinfo{author}{N.~Tait}, \bibinfo{author}{X.~Riera-Palou},
  \bibinfo{author}{A.~Harrison},
\newblock \bibinfo{title}{A new comparison between the life cycle greenhouse
  gas emissions of battery electric vehicles and internal combustion vehicles},
\newblock \bibinfo{journal}{Energy policy} \bibinfo{volume}{44}
  (\bibinfo{year}{2012}) \bibinfo{pages}{160--173}.
\bibitem[{Goodall et~al.(2017)Goodall, Dovey, Bornstein, and
  Bonthron}]{goodall2017rise}
\bibinfo{author}{W.~Goodall}, \bibinfo{author}{T.~Dovey},
  \bibinfo{author}{J.~Bornstein}, \bibinfo{author}{B.~Bonthron},
\newblock \bibinfo{title}{The rise of mobility as a service},
\newblock \bibinfo{journal}{Deloitte Rev} \bibinfo{volume}{20}
  (\bibinfo{year}{2017}) \bibinfo{pages}{112--129}.
\bibitem[{G{\"o}tz et~al.(2011)G{\"o}tz, Sunderer, Birzle-Harder, Deffner, and
  Berlin}]{gotz2011attraktivitat}
\bibinfo{author}{K.~G{\"o}tz}, \bibinfo{author}{G.~Sunderer},
  \bibinfo{author}{B.~Birzle-Harder}, \bibinfo{author}{J.~Deffner},
  \bibinfo{author}{B.~Berlin},
\newblock \bibinfo{title}{{Attraktivit\"{a}t und Akzeptanz von Elektroautos:
  Ergebnisse aus dem Projekt OPTUM Optimierung der Umweltentlastungspotenziale
  von Elektrofahrzeugen}},
\newblock \bibinfo{journal}{ISOE-Studientexte} \bibinfo{volume}{18}
  (\bibinfo{year}{2011}).
\bibitem[{Metcalfe and Dolan(2012)}]{metcalfe2012behavioural}
\bibinfo{author}{R.~Metcalfe}, \bibinfo{author}{P.~Dolan},
\newblock \bibinfo{title}{Behavioural economics and its implications for
  transport},
\newblock \bibinfo{journal}{Journal of transport geography}
  \bibinfo{volume}{24} (\bibinfo{year}{2012}) \bibinfo{pages}{503--511}.
\bibitem[{Ostermann et~al.(2014)Ostermann, Renner, Koetter, and
  Hudert}]{Ostermann2014}
\bibinfo{author}{J.~Ostermann}, \bibinfo{author}{T.~Renner},
  \bibinfo{author}{F.~Koetter}, \bibinfo{author}{S.~Hudert},
\newblock \bibinfo{title}{Leveraging electric cross-company car fleets through
  cloud service chains: The shared e-fleet architecture},
\newblock in: \bibinfo{booktitle}{Global Conference (SRII), 2014 Annual SRII},
  \bibinfo{organization}{IEEE}, \bibinfo{year}{2014}, pp.
  \bibinfo{pages}{290--297}.
\bibitem[{Koetter et~al.(2015)Koetter, Ostermann, and Jecan}]{Koetter2015}
\bibinfo{author}{F.~Koetter}, \bibinfo{author}{J.~Ostermann},
  \bibinfo{author}{D.~V. Jecan},
\newblock \bibinfo{title}{Schedule rating method based on a fragmentation
  criterion},
\newblock in: \bibinfo{booktitle}{The Fourth International Conference on
  Advances in Vehicular Systems, Technologies and Applications},
  \bibinfo{year}{2015}, p.~\bibinfo{pages}{28}.
\bibitem[{Betz et~al.(2016)Betz, Werner, and Lienkamp}]{Betz2016}
\bibinfo{author}{J.~Betz}, \bibinfo{author}{D.~Werner},
  \bibinfo{author}{M.~Lienkamp},
\newblock \bibinfo{title}{Fleet disposition modeling to maximize utilization of
  battery electric vehicles in companies with on-site energy generation},
\newblock \bibinfo{journal}{Transportation Research Procedia}
  \bibinfo{volume}{19} (\bibinfo{year}{2016}) \bibinfo{pages}{241--257}.
\bibitem[{Sassi and Oulamara(2017)}]{Sassi2014}
\bibinfo{author}{O.~Sassi}, \bibinfo{author}{A.~Oulamara},
\newblock \bibinfo{title}{Electric vehicle scheduling and optimal charging
  problem: complexity, exact and heuristic approaches},
\newblock \bibinfo{journal}{International Journal of Production Research}
  \bibinfo{volume}{55} (\bibinfo{year}{2017}) \bibinfo{pages}{519--535}.
\bibitem[{Kolen et~al.(2007)Kolen, Lenstra, Papadimitriou, and
  Spieksma}]{Kolen2007}
\bibinfo{author}{A.~W. Kolen}, \bibinfo{author}{J.~K. Lenstra},
  \bibinfo{author}{C.~H. Papadimitriou}, \bibinfo{author}{F.~C. Spieksma},
\newblock \bibinfo{title}{Interval scheduling: A survey},
\newblock \bibinfo{journal}{Naval Research Logistics (NRL)}
  \bibinfo{volume}{54} (\bibinfo{year}{2007}) \bibinfo{pages}{530--543}.
\bibitem[{Kovalyov et~al.(2007)Kovalyov, Ng, and Cheng}]{kovalyov2007fixed}
\bibinfo{author}{M.~Y. Kovalyov}, \bibinfo{author}{C.~Ng},
  \bibinfo{author}{T.~E. Cheng},
\newblock \bibinfo{title}{Fixed interval scheduling: Models, applications,
  computational complexity and algorithms},
\newblock \bibinfo{journal}{European journal of operational research}
  \bibinfo{volume}{178} (\bibinfo{year}{2007}) \bibinfo{pages}{331--342}.
\bibitem[{Kroon et~al.(1997)Kroon, Salomon, and
  Van~Wassenhove}]{kroon1997exact}
\bibinfo{author}{L.~G. Kroon}, \bibinfo{author}{M.~Salomon},
  \bibinfo{author}{L.~N. Van~Wassenhove},
\newblock \bibinfo{title}{Exact and approximation algorithms for the tactical
  fixed interval scheduling problem},
\newblock \bibinfo{journal}{Operations Research} \bibinfo{volume}{45}
  (\bibinfo{year}{1997}) \bibinfo{pages}{624--638}.
\bibitem[{Ng et~al.(2014)Ng, Cheng, Bandalouski, Kovalyov, and
  Lam}]{ng2014graph}
\bibinfo{author}{C.~T. Ng}, \bibinfo{author}{T.~C.~E. Cheng},
  \bibinfo{author}{A.~M. Bandalouski}, \bibinfo{author}{M.~Y. Kovalyov},
  \bibinfo{author}{S.~S. Lam},
\newblock \bibinfo{title}{A graph-theoretic approach to interval scheduling on
  dedicated unrelated parallel machines},
\newblock \bibinfo{journal}{Journal of the Operational Research Society}
  \bibinfo{volume}{65} (\bibinfo{year}{2014}) \bibinfo{pages}{1571--1579}.
\bibitem[{Bunte and Kliewer(2009)}]{bunte2009overview}
\bibinfo{author}{S.~Bunte}, \bibinfo{author}{N.~Kliewer},
\newblock \bibinfo{title}{An overview on vehicle scheduling models},
\newblock \bibinfo{journal}{Public Transport} \bibinfo{volume}{1}
  (\bibinfo{year}{2009}) \bibinfo{pages}{299--317}.
\bibitem[{Hassold and Ceder(2014)}]{hassold2014public}
\bibinfo{author}{S.~Hassold}, \bibinfo{author}{A.~A. Ceder},
\newblock \bibinfo{title}{Public transport vehicle scheduling featuring
  multiple vehicle types},
\newblock \bibinfo{journal}{Transportation Research Part B: Methodological}
  \bibinfo{volume}{67} (\bibinfo{year}{2014}) \bibinfo{pages}{129--143}.
\bibitem[{Desfontaines and Desaulniers(2018)}]{desfontaines2018multiple}
\bibinfo{author}{L.~Desfontaines}, \bibinfo{author}{G.~Desaulniers},
\newblock \bibinfo{title}{Multiple depot vehicle scheduling with controlled
  trip shifting},
\newblock \bibinfo{journal}{Transportation Research Part B: Methodological}
  \bibinfo{volume}{113} (\bibinfo{year}{2018}) \bibinfo{pages}{34--53}.
\bibitem[{Ernst et~al.(2011)Ernst, Gavriliouk, and Marquez}]{Ernst2011}
\bibinfo{author}{A.~T. Ernst}, \bibinfo{author}{E.~O. Gavriliouk},
  \bibinfo{author}{L.~Marquez},
\newblock \bibinfo{title}{An efficient lagrangean heuristic for rental vehicle
  scheduling},
\newblock \bibinfo{journal}{Computers \& Operations Research}
  \bibinfo{volume}{38} (\bibinfo{year}{2011}) \bibinfo{pages}{216--226}.
\bibitem[{Ernst et~al.(2007)Ernst, Horn, Krishnamoorthy, Kilby, Degenhardt, and
  Moran}]{Ernst2007}
\bibinfo{author}{A.~T. Ernst}, \bibinfo{author}{M.~Horn},
  \bibinfo{author}{M.~Krishnamoorthy}, \bibinfo{author}{P.~Kilby},
  \bibinfo{author}{P.~Degenhardt}, \bibinfo{author}{M.~Moran},
\newblock \bibinfo{title}{Static and dynamic order scheduling for recreational
  rental vehicles at tourism holdings limited},
\newblock \bibinfo{journal}{Interfaces} \bibinfo{volume}{37}
  (\bibinfo{year}{2007}) \bibinfo{pages}{334--341}.
\bibitem[{Oliveira et~al.(2014)Oliveira, Carravilla, Oliveira, and
  Toledo}]{oliveira2014relax}
\bibinfo{author}{B.~B. Oliveira}, \bibinfo{author}{M.~A. Carravilla},
  \bibinfo{author}{J.~F. Oliveira}, \bibinfo{author}{F.~M. Toledo},
\newblock \bibinfo{title}{A relax-and-fix-based algorithm for the
  vehicle-reservation assignment problem in a car rental company},
\newblock \bibinfo{journal}{European Journal of Operational Research}
  \bibinfo{volume}{237} (\bibinfo{year}{2014}) \bibinfo{pages}{729--737}.
\bibitem[{Farber(1984)}]{Farber1984}
\bibinfo{author}{M.~Farber},
\newblock \bibinfo{title}{Domination, independent domination, and duality in
  strongly chordal graphs},
\newblock \bibinfo{journal}{Discrete Applied Mathematics} \bibinfo{volume}{7}
  (\bibinfo{year}{1984}) \bibinfo{pages}{115--130}.
\bibitem[{Escoffier et~al.(2005)Escoffier, Monnot, and Paschos}]{Escoffier2005}
\bibinfo{author}{B.~Escoffier}, \bibinfo{author}{J.~Monnot},
  \bibinfo{author}{V.~T. Paschos},
\newblock \bibinfo{title}{Weighted coloring: further complexity and
  approximability results},
\newblock in: \bibinfo{booktitle}{Italian Conference on Theoretical Computer
  Science}, \bibinfo{organization}{Springer}, \bibinfo{year}{2005}, pp.
  \bibinfo{pages}{205--214}.
\bibitem[{Bar-Noy et~al.(2001)Bar-Noy, Bar-Yehuda, Freund, Naor, and
  Schieber}]{Bar2001}
\bibinfo{author}{A.~Bar-Noy}, \bibinfo{author}{R.~Bar-Yehuda},
  \bibinfo{author}{A.~Freund}, \bibinfo{author}{J.~Naor},
  \bibinfo{author}{B.~Schieber},
\newblock \bibinfo{title}{A unified approach to approximating resource
  allocation and scheduling},
\newblock \bibinfo{journal}{Journal of the ACM (JACM)} \bibinfo{volume}{48}
  (\bibinfo{year}{2001}) \bibinfo{pages}{1069--1090}.
\bibitem[{Angelelli et~al.(2014)Angelelli, Bianchessi, and
  Filippi}]{Angelelli2014}
\bibinfo{author}{E.~Angelelli}, \bibinfo{author}{N.~Bianchessi},
  \bibinfo{author}{C.~Filippi},
\newblock \bibinfo{title}{Optimal interval scheduling with a resource
  constraint},
\newblock \bibinfo{journal}{Computers \& Operations Research}
  \bibinfo{volume}{51} (\bibinfo{year}{2014}) \bibinfo{pages}{268--281}.
\bibitem[{Butman et~al.(2010)Butman, Hermelin, Lewenstein, and
  Rawitz}]{Butman2010}
\bibinfo{author}{A.~Butman}, \bibinfo{author}{D.~Hermelin},
  \bibinfo{author}{M.~Lewenstein}, \bibinfo{author}{D.~Rawitz},
\newblock \bibinfo{title}{Optimization problems in multiple-interval graphs},
\newblock \bibinfo{journal}{ACM Transactions on Algorithms (TALG)}
  \bibinfo{volume}{6} (\bibinfo{year}{2010}) \bibinfo{pages}{40}.
\bibitem[{Burke et~al.(2010)Burke, Mare{\v{c}}ek, Parkes, and
  Rudov{\'a}}]{Burke2010}
\bibinfo{author}{E.~K. Burke}, \bibinfo{author}{J.~Mare{\v{c}}ek},
  \bibinfo{author}{A.~J. Parkes}, \bibinfo{author}{H.~Rudov{\'a}},
\newblock \bibinfo{title}{A supernodal formulation of vertex colouring with
  applications in course timetabling},
\newblock \bibinfo{journal}{Annals of Operations Research}
  \bibinfo{volume}{179} (\bibinfo{year}{2010}) \bibinfo{pages}{105--130}.
\bibitem[{Marx(2004)}]{Marx2004}
\bibinfo{author}{D.~Marx},
\newblock \bibinfo{title}{Graph colouring problems and their applications in
  scheduling},
\newblock \bibinfo{journal}{Periodica Polytechnica Electrical Engineering}
  \bibinfo{volume}{48} (\bibinfo{year}{2004}) \bibinfo{pages}{11--16}.
\bibitem[{Bast et~al.(2016)Bast, Delling, Goldberg, M{\"u}ller-Hannemann,
  Pajor, Sanders, Wagner, and Werneck}]{Bast2016}
\bibinfo{author}{H.~Bast}, \bibinfo{author}{D.~Delling},
  \bibinfo{author}{A.~Goldberg}, \bibinfo{author}{M.~M{\"u}ller-Hannemann},
  \bibinfo{author}{T.~Pajor}, \bibinfo{author}{P.~Sanders},
  \bibinfo{author}{D.~Wagner}, \bibinfo{author}{R.~F. Werneck},
\newblock \bibinfo{title}{Route planning in transportation networks},
\newblock in: \bibinfo{booktitle}{Algorithm Engineering},
  \bibinfo{publisher}{Springer International Publishing}, \bibinfo{year}{2016},
  pp. \bibinfo{pages}{19--80}.
\bibitem[{Arkin and Silverberg(1987)}]{Arkin1987}
\bibinfo{author}{E.~M. Arkin}, \bibinfo{author}{E.~B. Silverberg},
\newblock \bibinfo{title}{Scheduling jobs with fixed start and end times},
\newblock \bibinfo{journal}{Discrete Applied Mathematics} \bibinfo{volume}{18}
  (\bibinfo{year}{1987}) \bibinfo{pages}{1--8}.
\bibitem[{Moon and Moser(1965)}]{Moon1965}
\bibinfo{author}{J.~W. Moon}, \bibinfo{author}{L.~Moser},
\newblock \bibinfo{title}{On cliques in graphs},
\newblock \bibinfo{journal}{Israel journal of Mathematics} \bibinfo{volume}{3}
  (\bibinfo{year}{1965}) \bibinfo{pages}{23--28}.
\bibitem[{Gupta et~al.(1982)Gupta, Lee, and Leung}]{Gupta1982}
\bibinfo{author}{U.~I. Gupta}, \bibinfo{author}{D.-T. Lee},
  \bibinfo{author}{J.-T. Leung},
\newblock \bibinfo{title}{Efficient algorithms for interval graphs and
  circular-arc graphs},
\newblock \bibinfo{journal}{Networks} \bibinfo{volume}{12}
  (\bibinfo{year}{1982}) \bibinfo{pages}{459--467}.
\bibitem[{Gupta et~al.(1979)Gupta, Lee, and Leung}]{Gupta1979}
\bibinfo{author}{U.~I. Gupta}, \bibinfo{author}{D.-T. Lee},
  \bibinfo{author}{J.-T. Leung},
\newblock \bibinfo{title}{An optimal solution for the channel-assignment
  problem},
\newblock \bibinfo{journal}{IEEE Transactions on Computers}
  (\bibinfo{year}{1979}) \bibinfo{pages}{807--810}.
\bibitem[{Gustin(1963)}]{Gustin1963}
\bibinfo{author}{W.~Gustin},
\newblock \bibinfo{title}{Orientable embedding of cayley graphs},
\newblock \bibinfo{journal}{Bulletin of the American Mathematical Society}
  \bibinfo{volume}{69} (\bibinfo{year}{1963}) \bibinfo{pages}{272--275}.
\bibitem[{Sanders and Schulz(2012)}]{Schulz2013}
\bibinfo{author}{P.~Sanders}, \bibinfo{author}{C.~Schulz},
\newblock \bibinfo{title}{High quality graph partitioning.},
\newblock \bibinfo{journal}{Graph Partitioning and Graph Clustering}
  \bibinfo{volume}{588} (\bibinfo{year}{2012}) \bibinfo{pages}{1--17}.
\bibitem[{Shaw(1998)}]{Shaw1998}
\bibinfo{author}{P.~Shaw},
\newblock \bibinfo{title}{Using constraint programming and local search methods
  to solve vehicle routing problems},
\newblock in: \bibinfo{booktitle}{International Conference on Principles and
  Practice of Constraint Programming}, \bibinfo{organization}{Springer},
  \bibinfo{year}{1998}, pp. \bibinfo{pages}{417--431}.
\bibitem[{Pisinger and Ropke(2010)}]{Pisinger2010}
\bibinfo{author}{D.~Pisinger}, \bibinfo{author}{S.~Ropke},
\newblock \bibinfo{title}{Large neighborhood search},
\newblock in: \bibinfo{booktitle}{Handbook of metaheuristics},
  \bibinfo{publisher}{Springer}, \bibinfo{year}{2010}, pp.
  \bibinfo{pages}{399--419}.
\bibitem[{Ropke and Pisinger(2006)}]{ropke2006adaptive}
\bibinfo{author}{S.~Ropke}, \bibinfo{author}{D.~Pisinger},
\newblock \bibinfo{title}{An adaptive large neighborhood search heuristic for
  the pickup and delivery problem with time windows},
\newblock \bibinfo{journal}{Transportation science} \bibinfo{volume}{40}
  (\bibinfo{year}{2006}) \bibinfo{pages}{455--472}.
\bibitem[{Prescott-Gagnon et~al.(2009)Prescott-Gagnon, Desaulniers, and
  Rousseau}]{Prescott-Gagnon2009}
\bibinfo{author}{E.~Prescott-Gagnon}, \bibinfo{author}{G.~Desaulniers},
  \bibinfo{author}{L.-M. Rousseau},
\newblock \bibinfo{title}{A branch-and-price-based large neighborhood search
  algorithm for the vehicle routing problem with time windows},
\newblock \bibinfo{journal}{Networks} \bibinfo{volume}{54}
  (\bibinfo{year}{2009}) \bibinfo{pages}{190--204}.
\bibitem[{Ribeiro and Laporte(2012)}]{Ribeiro2012}
\bibinfo{author}{G.~M. Ribeiro}, \bibinfo{author}{G.~Laporte},
\newblock \bibinfo{title}{An adaptive large neighborhood search heuristic for
  the cumulative capacitated vehicle routing problem},
\newblock \bibinfo{journal}{Computers \& operations research}
  \bibinfo{volume}{39} (\bibinfo{year}{2012}) \bibinfo{pages}{728--735}.
\bibitem[{Ropke and Pisinger(2006)}]{Ropke2006}
\bibinfo{author}{S.~Ropke}, \bibinfo{author}{D.~Pisinger},
\newblock \bibinfo{title}{An adaptive large neighborhood search heuristic for
  the pickup and delivery problem with time windows},
\newblock \bibinfo{journal}{Transportation science} \bibinfo{volume}{40}
  (\bibinfo{year}{2006}) \bibinfo{pages}{455--472}.
\bibitem[{Godard et~al.(2005)Godard, Laborie, and Nuijten}]{Godard2005}
\bibinfo{author}{D.~Godard}, \bibinfo{author}{P.~Laborie},
  \bibinfo{author}{W.~Nuijten},
\newblock \bibinfo{title}{Randomized large neighborhood search for cumulative
  scheduling.},
\newblock in: \bibinfo{booktitle}{International Conference on Automated
  Planning and Scheduling (ICAPS)}, volume~\bibinfo{volume}{5},
  \bibinfo{year}{2005}, pp. \bibinfo{pages}{81--89}.
\bibitem[{Feo and Resende(1995)}]{Feo1995}
\bibinfo{author}{T.~A. Feo}, \bibinfo{author}{M.~G. Resende},
\newblock \bibinfo{title}{Greedy randomized adaptive search procedures},
\newblock \bibinfo{journal}{Journal of global optimization} \bibinfo{volume}{6}
  (\bibinfo{year}{1995}) \bibinfo{pages}{109--133}.
\bibitem[{Lueker and Booth(1979)}]{lueker1979}
\bibinfo{author}{G.~S. Lueker}, \bibinfo{author}{K.~S. Booth},
\newblock \bibinfo{title}{A linear time algorithm for deciding interval graph
  isomorphism},
\newblock \bibinfo{journal}{Journal of the ACM (JACM)} \bibinfo{volume}{26}
  (\bibinfo{year}{1979}) \bibinfo{pages}{183--195}.
\bibitem[{L{\'o}pez-Ib{\'a}{\~n}ez et~al.(2016)L{\'o}pez-Ib{\'a}{\~n}ez,
  Dubois-Lacoste, C{\'a}ceres, Birattari, and St{\"u}tzle}]{lopez2016irace}
\bibinfo{author}{M.~L{\'o}pez-Ib{\'a}{\~n}ez},
  \bibinfo{author}{J.~Dubois-Lacoste}, \bibinfo{author}{L.~P. C{\'a}ceres},
  \bibinfo{author}{M.~Birattari}, \bibinfo{author}{T.~St{\"u}tzle},
\newblock \bibinfo{title}{The irace package: Iterated racing for automatic
  algorithm configuration},
\newblock \bibinfo{journal}{Operations Research Perspectives}
  \bibinfo{volume}{3} (\bibinfo{year}{2016}) \bibinfo{pages}{43--58}.
\bibitem[{Enzi et~al.(2020)Enzi, Parragh, and Pisinger}]{enzi2020modeling}
\bibinfo{author}{M.~Enzi}, \bibinfo{author}{S.~N. Parragh},
  \bibinfo{author}{D.~Pisinger},
\newblock \bibinfo{title}{Modeling and solving a vehicle-sharing problem},
\newblock \bibinfo{journal}{arXiv preprint arXiv:2003.08207}
  (\bibinfo{year}{2020}).

\end{thebibliography}

\clearpage
\appendix

\section{Detailed description of the generation of instance set RW}\label{app:instance}

The benchmark instances are based on demographic, spatial, and economic data of Vienna, Austria,
and consider a company which operates in that area.
First, a set of mode of transport classes $K$ are defined consisting of the following types: 
\emph{Foot}, \emph{Public transport},	\emph{Bike}, \emph{Battery electric vehicle (BEV)} and subtypes corresponding to specific car models, \emph{Internal combustion engine vehicle (ICEV)} and subtypes corresponding to the size of the vehicle and \emph{Taxi}.
For both BEV and ICEV several sub-categories are defined which correspond to car models, e.g., for BEVs we consider Smart ED, Nissan Leaf, and Mitsubishi iMiev.
The properties of each $k\in K$ are the following:

\bigskip

\begin{center}
\begin{tabular}{lllp{25em}}
Parameter & Domain & Unit & Description \\ \midrule
$\epsilon^k$ & $\mathbb{R}$ & g/km &  CO$_2$ emissions per distance unit \\
$v^k$ & $\mathbb{R}$ & m/s & average speed \\
$c^k_d$ & $\mathbb{R}$ & 1/km & cost in Euro per distance \\
$c^k_t$ & $\mathbb{R}$ & 1/min & cost in Euro per time \\
$a^k$ & $\mathbb{R}$ & s & additional time needed for setup (e.g., getting to the car, time needed for parking)
\end{tabular}
\end{center}

\bigskip

Then, a company is constructed consisting of one or more depots $\Delta\subset L$, where each $\delta\in \Delta$ is represented by its geographic coordinate and $L$ is the set of all possible locations.
The company has a set of employees $P$, and a number of available instances $n_k$ of each transport class $k\in K$.
Note that $n_k=\infty$ for \emph{foot}, \emph{public transport}, and \emph{taxi}.

Each employee $p\in P$ has a gender $\theta^p\in \{\mathrm{f},\mathrm{m}\}$, a hierarchy status $h^p\in \{\mathrm{b,m,w}\}$ (boss, middle management, worker), an associated office location $\delta^p\in \Delta$, a home location $l^p\in L$, a work start time $\tau^s \in \mathbb{N}$, and a work end time $\tau^e \in \mathbb{N}$
For all $k\in K$ it is specified if employee $p$ is willing to accept offers using transport mode $k$, denoted by $\omega^{pk}\in\{0,1\}$, $\forall k\in K, p\in P$.

Then, for each employee $p\in P$ on each day $t\in T$ of the considered time horizon $T$ an ordered list of events $E^{pt}=(e^{pt}_{0},\dots,e^{pt}_{n})$ is generated (representing a working day of this employee) consisting of the following attributes:

\bigskip

\begin{center}
\begin{tabular}{lllp{25em}}
Parameter & Domain & Unit & Description \\ \midrule
$\alpha^e$ & $\mathbb{N}$ & min & latest arrival (in number of minutes from the start of the time horizon) \\
$\beta^e$ & $\mathbb{N}$ & min & earliest departure \\
$s^e$ & $\mathbb{N}$ & min & service duration \\
$l^e$ & $L$ & & location \\
$t^e$ & $\{\mathrm{w,m,p,h}\}$ & & activity type: \emph{work}, \emph{meeting}, \emph{private}, \emph{home}
\end{tabular}
\end{center}

\bigskip

Furthermore, for each pair of locations $l_1,l_2\in L$ a distance $d_{ij}^k$, travel time $t_{ij}^k$, and cost matrix $c_{ij}^k$ is computed for each $k\in K$ based on the route from $l_1$ to $l_2$ in the road network.

\paragraph{Value Settings}~\newline
This section describes how the independent values of the variables described above are set.
Some of the variables are chosen randomly following the stated probability distribution.
In these cases the actual instance is generated by drawing one sample of each of these distributions.

\noindent\emph{Transport classes: }

\bigskip

\begin{center}
\begin{tabular}{lllp{25em}}
Parameter & Variability & Scope & Value \\ \midrule
$\epsilon^k$ & fixed & all & average values of the respective car category \\
$v^k$ & fixed & all & foot: 5, bike: 16, car: 30, public transport: 20 [km/h] \\
$c^k_d$ & fixed & all &  total cost of ownership divided by total km \\
$c^k_t$ & fixed & all &  average gross salary in Austria including additional costs for employer \\
$a^k$ & fixed & all & foot: 0, bike: 120, car: 600, public transport: 300, taxi: 300 [s]
\end{tabular}
\end{center}

\bigskip

\noindent\emph{Company: }

\begin{center}
\begin{tabular}{llp{25em}}
Parameter & Variability & Value \\ \midrule
$L$ & fixed & geometric centers of all 250 registration districts of Vienna \\
$T$ & fixed & one week \\
$\Delta$ & fixed & two locations chosen randomly following the probability distribution $\mathcal{P}^o$ of $L$, where $\mathcal{P}^o$ is based on statistical data of office locations in Vienna \\
$|P|$ & variable & integer value \\
$\nu$ & variable & real value in the interval [0,1] determining $n_k$, $\forall k\in K$ \\
$n_k$ & fixed & for bikes, BEVs, and ICEVs: between 0 and $\lfloor\nu|P|\rfloor$ \\
\end{tabular}
\end{center}
\bigskip

\noindent\emph{Employee: }

\begin{center}
\begin{tabular}{llp{25em}}
Parameter & Variability & Value \\ \midrule
$\theta^p$ & fixed & based on demographic data of female and male employees (f: 46.78\%, m: 53.22\%) \\
$h^p$ & fixed & $P(h^p=\mathrm{b})=0.01$, $P(h^p=\mathrm{m})=0.1$, $P(h^p=\mathrm{w})=0.89$\\
$\delta^p$ & fixed & chosen uniformly at random out of $\Delta$ \\
$l^p$ & fixed & chosen randomly following the probability distribution $\mathcal{P}^h$ of $L$, where $\mathcal{P}^h$ is based on statistical data of residential locations in Vienna \\
$\tau^s$ & fixed & chosen randomly following a probability distribution $\mathcal{P}^{\tau^s}$ between 5 and 11 a.m.  \\
$\tau^e$ & fixed & $\tau^s$ + amount of daily working hours, which are chosen randomly following a probability distribution $\mathcal{P}^{\tau^e}$ which depends on $\theta^p$ and $h^p$ \\
$\omega^{pk}$ & fixed & we defined 7 combinations of accepted mode of transports, e.g., car only, public transport only, mixed. For each combination at most different acceptance scenarios are defined. The combinations are chosen randomly based on a probability distribution $\mathcal{P}^\omega$ considering gender and the probability that $p$ has a driving license which itself is based on statistical data. The acceptance scenario of the chosen category is taken uniformly at random.
\end{tabular}
\end{center}

\bigskip

\noindent\emph{Events: }

\begin{center}
\begin{tabular}{llp{25em}}
Parameter & Variability & Value \\ \midrule
$\alpha^e$ & fixed & private activity: at any time outside working hours. Work meeting: at any time within the working hours. \\
$\beta^e$ & fixed & $\alpha^e+s^e$ \\
$s^e$ & fixed & private meetings in the morning 60 minutes, in the evening 120 minutes. Work meetings between 30 and 180 minutes based on probability distribution $\mathcal{P}^{s^e}$.\\
$l^e$ & fixed & based on $\mathcal{P}^h$ for private activities, on $\mathcal{P}^o$ for work meetings \\
$t^e$ & fixed & for each day: private activity in the morning with 20\% probability, in the evening with 65\% probability. The number of work meetings is based on $h^p$ which results in a average amount of time spent in meetings. A meeting is inserted into the daily schedule of the employee until this time is spent or it does not fit in anymore. \\
\end{tabular}
\end{center}

\noindent\emph{Distance, travel time, and cost: }
\bigskip

\begin{center}
\begin{tabular}{llp{25em}}
Parameter & Variability & Value \\ \midrule
$d^k_{ij}$ & fixed & Aerial distance between $i$ and $j$ multiplied by a constant sloping factor of the respective mode of transport $k$ \\
$t^k_{ij}$ & fixed & $\frac{d^k_{ij}}{v^k}$ \\
$c^k_{ij}$ & fixed & $c^k_dd^k_{ij}+c^k_tt^k_{ij}+\epsilon^kc_e$, where $c_e$ are the CO$_2$ costs which are set to 5 Euro per ton.
\end{tabular}
\end{center}

\bigskip

\paragraph{Generation of the Mobility Offers}

Based on the data described above we extract mobility demands and offers which form the actual instance of our optimization problem.
First, we generate the set of mobility demands $D$ by considering the events $E^p=\bigcup_{t\in T}E^{pt}$ of each employee $p\in P$.
Since we assume that the company fleet is located at the depots $\Delta$, each mobility demand $d\in D$ consists of a tour starting and ending at the office location $\delta^p$ of the corresponding employee $p$.
Therefore we construct the set of demands $D^p=\{d^p_0,\dots,d^p_m\}$ with $d^p_i=(e^{p}_{j},e^p_{j+1},\dots,e^p_{q})\subseteq E^p$ with $q>j$ for all $j=0,\dots, n$ with $t^{e_j^p}=t^{e_q^p}=\textrm{w}$, $\forall i\in 0,\dots,m$ for each employee $p\in P$.

For each $p\in P$ and each $d\in D^p$ a set of mobility offers $O^{d}$ is created.
There is one offer for each transport class $k\in K$ which is accepted by the employee, i.e., for which $\omega^{pk}=1$, denoted by $k^o\in K$.
Each offer $o\in O^{d}$ has an \emph{journey interval} $[a_{o},b_{o}]$ with $a_{o},b_{o}\in \mathbb{R}$ defining its start time $a_{o}$ and end time $b_{o}$.
The start time $a_{o}$ is given by the latest arrival $\alpha^{e^p_{j+1}}$ of the first event of the associated demand subtracted by half the setup time $\frac{1}{2}a^{k}$ of the corresponding vehicle type $k\in K$ and the travel time $t_{l_1,l_2}^{k}$ with $l_1=l^{e^p_{j}}=\delta^p$ and $l_2=l^{e^p_{j+1}}$.
The end time is given by $b_{o}=\beta^{e^p_{q}}+t^{k}_{l_3,l_4}+\frac{1}{2}a^{k}$ with $l_3=l^{e^p_{q}}$ and $l_4=\delta^p$ resulting in a duration $\pi_{o^{d^p}}=b_{o^{d^p}}-a_{o^{d^p}}$.

Finally, the cost $c_{o}$ of each offer $o\in O^{d}$, $\forall d\in D^p, p\in P$ is generated based on the cost matrix of the relevant events and the corresponding transport class $k$.
The salary costs which depend on the duration of the offer are, however, only considered for \emph{work events}, i.e., the journeys from work to the meetings and from the meetings back to work.
More specifically, the cost of an offer $o\in O^{d}$ with $d^p=(e^{p}_{j},e^p_{j+1},\dots,e^p_{q})$, $\forall p\in P$ contains the setup costs $C_S$ and the travel costs $C_T$ is $C^o=C_S+C_T$ with:

$$C_S=a^{k}c_t^{k} $$

$$C_T= \sum\limits_{i=j}^{q-1}c_{l^{e^p_i}l^{e^p_{i+1}}}^{k} - 0.8t_{l^{e^p_i}l^{e^p_{i+1}}}c_t^{k}(1-\Gamma_{e^p_ie^p_{i+1}}), \text{with}$$

\begin{equation*}
\Gamma_{e^p_ie^p_{i+1}}=\left\{\begin{array}{ll}
1 & \text{if } (t^{e^p_i} = \mathrm{w} \land t^{e^p_i} = \mathrm{m}) \lor (t^{e^p_i} = \mathrm{m} \land t^{e^p_i} = \mathrm{w}) \\
0 & \text{else}
\end{array}\right.
\end{equation*}

\end{document}